\documentclass[11pt]{article}
\usepackage[a4paper]{geometry}
\usepackage[usenames,dvipsnames]{xcolor}
\usepackage{hyperref}

\usepackage{natbib}
\usepackage{float}
\usepackage{amssymb}
\usepackage{amsmath}
\usepackage{algorithm, algorithmic}
\usepackage{bbm}
\usepackage{mathtools}
\usepackage{makecell}
\usepackage{caption}
\usepackage{subcaption}
\usepackage{wrapfig}
\usepackage{authblk}
\newcommand{\bs}{\boldsymbol}
\DeclareMathOperator{\Tr}{Tr}

\newtheorem{theorem}{Theorem}
\newtheorem{lemma}[theorem]{Lemma}

\newenvironment{proof}{{\bf Proof:}}{\hfill\rule{2mm}{2mm}}
\newenvironment{remark}{{\bf Remark:}}

\renewcommand{\algorithmiccomment}[1]{\hfill$\triangleright$ #1}
\newcommand{\ewsg}[1]{{\color{RoyalBlue} #1}}
\newcommand{\SGHMC}[1]{{\color{RedOrange} #1}}

\begin{document}

\title{Improving Sampling Accuracy of Stochastic Gradient MCMC Methods via Non-uniform Subsampling of Gradients}
\author[1]{Ruilin Li}
\author[2]{Xin Wang}
\author[3]{Hongyuan Zha}
\author[1]{Molei Tao\thanks{\url{mtao@gatech.edu}}}
\affil[1]{Georgia Institute of Technology}
\affil[2]{Google Inc.}
\affil[3]{The Chinese University of Hong Kong, Shenzhen}
\date{}

\maketitle

\begin{abstract}
Many Markov Chain Monte Carlo (MCMC) methods leverage gradient information of the potential function of target distribution to explore sample space efficiently. However, computing gradients can often be computationally expensive for large scale applications, such as those in contemporary machine learning. Stochastic Gradient (SG-)MCMC methods approximate gradients by stochastic ones, commonly via uniformly subsampled data points, and achieve improved computational efficiency, however at the price of introducing sampling error. We propose a non-uniform subsampling scheme to improve the sampling accuracy. The proposed exponentially weighted stochastic gradient (EWSG) is designed so that a non-uniform-SG-MCMC method mimics the statistical behavior of a batch-gradient-MCMC method, and hence the inaccuracy due to SG approximation is reduced. EWSG differs from classical variance reduction (VR) techniques as it focuses on the entire distribution instead of just the variance; nevertheless, its reduced local variance is also proved. EWSG can also be viewed as an extension of the importance sampling idea, successful for stochastic-gradient-based optimizations, to sampling tasks. 
In our practical implementation of EWSG, the non-uniform subsampling is performed efficiently via a Metropolis-Hastings chain on the data index, which is coupled to the MCMC algorithm. Numerical experiments are provided, not only to demonstrate EWSG's effectiveness, but also to guide hyperparameter choices, and validate our \emph{non-asymptotic global error bound} despite of approximations in the implementation. Notably, while statistical accuracy is improved, convergence speed can be comparable to the uniform version, which renders EWSG a practical alternative to VR (but EWSG and VR can be combined too).
\end{abstract}

\section{Introduction} \label{introduction}
Consider the construction of algorithms that sample a target probability distribution $\pi\sim Z^{-1} \rho(\bs{x}) d\bs{x}$, where $Z$ is a normalization constant and the unnormalized density $\rho$ is assumed to be nonzero on the domain. Let $V(\bs{x}):=-\log \rho(\bs{x})$ and then the target density can be rewritten in the form of Gibbs distribution, i.e. $Z^{-1} \exp(-V(\bs{x}))$, where $V$ will be referred to as the potential function. 

For this purpose, many MCMC algorithms use physics-inspired evolution such as Langevin dynamics \citep{brooks2011handbook} to utilize gradient information (i.e., $\nabla V$) in order to efficiently explore the target distribution over continuous parameter space. However, gradient-based MCMC methods are often limited by the computational cost of evaluating the gradient on large data sets, which often correspond to specific potentials of the form $V(\bs{x}) = \sum_{i=1}^n V_i(\bs{x})$, where $n$ is very large; this type of additive potential with many terms will be the setup of this paper.

Motivated by the great success of stochastic gradient methods for optimization, which uses a stochastic estimator of the batch gradient $\nabla V$ instead of evaluating all $\nabla V_i$ terms, stochastic gradient MCMC methods (SG-MCMC) for sampling have also been gaining increasing attention. More precisely, when the accurate but expensive-to-evaluate batch gradients in a MCMC method are replaced by computationally cheaper estimates based on a subset of the data, the method is turned to a stochastic gradient version. Classical examples include SG (overdamped) Langevin Dynamics \citep{welling2011bayesian} and SG Hamiltonian Monte Carlo \citep{chen2014stochastic}, both designed for scalability suitable for machine learning tasks.

However, directly replacing the batch gradient by a (uniform) stochastic one without additional mitigation generally causes a MCMC method to sample from a statistical distribution different from the target, because the transition kernel of the MCMC method gets corrupted by the noise of subsampled gradient. In general, the additional noise is tolerable if the learning rate/step size is tiny or decreasing. However, when large steps are used for better efficiency, the extra noise is non-negligible  and undermines the performance of downstream applications such as Bayesian inference. 

In this paper, we present a state-dependent non-uniform SG-MCMC algorithm termed \textbf{E}xponentially \textbf{W}eighted \textbf{S}tochastic \textbf{G}radients method (EWSG), which continues the efforts of uniform SG-MCMC methods for scalable sampling. 
Our approach is based on designing the transition kernel of a SG-MCMC method to approximate the transition kernel of a full-gradient-based MCMC method. This approximation leads to non-uniform (in fact, exponential) weights that aim at capturing the entire state-variable distribution of the full-gradient-based MCMC method, rather than providing unbiased gradient estimator and reducing its variance. Nevertheless, if focusing on the variance, the advantage of EWSG is the following: recall the stochasticity of a SG-MCMC method can be decomposed into the \textit{intrinsic} randomness of MCMC and the \textit{extrinsic} randomness introduced by gradient subsampling; in conventional uniform subsampling treatments, the latter randomness is independent of the former, and thus when they are coupled together, variances add up; EWSG, on the other hand, dynamically chooses the weight of each datum according to the current state of the MCMC, and thus the variances do not add up due to dependence. However, the gained accuracy is beyond reduced variance, as EWSG, when converged, samples from a distribution close to the invariant distribution of the full-gradient MCMC method (which has no variance contributed by the extrinsic randomness), because its transition kernel (of the corresponding Markov process) is close to that of the full-gradient-MCMC method. 
This is how better sampling accuracy can be achieved.

Our main demonstration of EWSG is based on 2nd-order Langevin equations (a.k.a. inertial, kinetic, or underdamped Langevin), although it works for other MCMC methods too (e.g., Appendix \ref{ewsg:sgld},\ref{variance reduction}). To concentrate on the role of non-uniform SG weights, we will work with constant step sizes only. The fact that EWSG has locally reduced variance than its uniform counterpart is rigorously shown in Theorem \ref{thm:local}. Furthermore, a global non-asymptotic error analysis is given in Theorem \ref{thm:mse} to quantify the convergence and improved accuracy of EWSG, as well as to provide insights about hyperparameter choices.

Practically, the non-uniform gradient subsampling of EWSG is efficiently implemented via a Metropolis-Hastings chain over the data index. A number of experiments on synthetic and real world data sets, across downstream tasks including Bayesian logistic regression and Bayesian neural networks, are conducted to demonstrate the effectiveness of EWSG and validate our theoretical results, despite the approximation used in the implementation. In addition to improved accuracy, the convergence speed was empirically observed, in a fair comparison setup based on the same data pass, to be comparable to its uniform counterpart when hyper-parameters are appropriately chosen. The convergence (per data pass) was also seen to be clearly faster than a classical Variance Reduction (VR) approach (note: for sampling, not optimization), and EWSG hence provides a useful alternative to VR. Additional theoretical study of EWSG convergence speed is provided in Appendix \ref{sec:speed}.

Notation-wise, $\nabla V$ will be referred to as the full/batch-gradient, $n \nabla V_I$ with random $I \in [n]$, which is a statistical estimator of $\nabla V$, will be called stochastic gradient (SG), and when $I$ is uniformly distributed it will be called a uniform SG/subsampling, otherwise non-uniform. When uniform SG is used to approximate the batch-gradient in underdamped Langevin, the method will be referred to as (vanilla) Stochastic Gradient Underdamped Langevin Dynamics (SGULD/SGHMC\footnote{SGULD is the same as the well-known SGHMC with $\hat{B} = 0$, see eq. (13) and Sec. 3.3 in \cite{chen2014stochastic} for details. To be consistent with existing literature, we will refer SGULD as SGHMC in the sequel.}), and it serves as a baseline in experiments. 





\section{Related Works}\label{related}
\paragraph{Stochastic Gradient MCMC Methods (SG-MCMC)}
Based on approximating gradients by uniformly subsampled ones, stochastic gradient methods are computationally more favorable than their full gradient counterparts and have been widely studied and used in the field of optimization. Inspired by the great success of stochastic gradient methods in optimization, people also have also applied stochastic gradient methods to sampling problems. Since the seminal work of Stochastic Gradient Langevin Dynamics (SGLD) \citep{welling2011bayesian}, much progress \citep{ahn2012bayesian, patterson2013stochastic} has been made in the field of SG-MCMC.
\cite{teh2016consistency} theoretically justified the convergence of SGLD and offered practical guidance on tuning step size. \cite{li2016preconditioned} introduced a preconditioner and improved stability of SGLD. We also refer to \cite{maclaurin2015firefly} and \cite{fu2017cpsg} which will be discussed in Sec. \ref{experiments}. While these work were mostly based on 1st-order (overdamped) Langevin, other dynamics were considered too. For instance, \cite{chen2014stochastic} proposed Stochastic Gradient Hamiltonian Monte Carlo (SGHMC), which is closely related to 2nd-order Langevin dynamics \citep{bou2018geometric,bou2018coupling}, and \cite{ma2015complete} put it in a more general framework. 2nd-order Langevin was recently shown to be faster than the 1st-order version in appropriate setups \citep{cheng2017underdamped,cheng2018sharp,li2021mean} and began to gain more attention.

\paragraph{Variance Reduction (VR)} 
For \underline{optimization}, vanilla SG methods usually find approximate solutions quickly but the convergence slows down (due to variance) when an accurate solution is needed \citep{bach2013stochastic,johnson2013accelerating}. SAG \citep{schmidt2013minimizing} improved the convergence speed of stochastic gradient methods to linear, which is the same as gradient descent methods with full gradient, at the expense of large memory overhead. SVRG \citep{johnson2013accelerating} successfully reduced this memory overhead. SAGA \citep{defazio2014saga} furthers improved convergence speed over SAG and SVRG. For \underline{sampling}, \cite{dubey2016variance} applied VR techniques to SGLD (see also \citep{baker2019control,chatterji2018theory}). However, many VR methods have large memory overhead and/or periodically use the whole data set for gradient estimation calibration, and hence can be resource-demanding.

EWSG is derived based on matching transition kernels of MCMC and improves the accuracy of the entire distribution rather than just the variance. However, it does have a consequence of variance reduction and thus can be implicitly regarded as a VR method. When compared to the classic work on VR for SG-MCMC \citep{dubey2016variance}, EWSG converges faster when the same amount of data pass is used, although its sampling accuracy is below that of VR for Gaussian targets (but well above vanilla SG; see Sec. \ref{subsec:simple_gaussian}). In this sense, EWSG and VR suit different application domains: EWSG can replace vanilla SG for tasks in which the priority is speed and then accuracy, as it keeps the speed but improves the accuracy; on the other hand, VR remains to be the heavy weapon for accuracy-demanding scenarios. Importantly, EWSG, as a generic way to improve SG-MCMC methods, can be combined with VR too (e.g., Sec. \ref{variance reduction}); thus, they are not exclusive or competing with each other.


\paragraph{Importance Sampling (IS)}
IS methods employ nonuniform weights to improve the convergence speed of stochastic gradient methods for \underline{optimization}. Traditional IS methods use fixed weights that do not change along iterations, and the weight computation requires prior information of gradient terms, e.g., Lipschitz constant of the gradient \citep{needell2014stochastic,schmidt2015non,csiba2018importance}, which are usually unknown or difficult to estimate. Adaptive IS was also  proposed in which the importance was re-evaluated at each iteration, whose computation usually required the entire data set per iteration and may also require information like the upper bound of gradient \citep{zhao2015stochastic,zhu2016gradient}.

For \underline{sampling}, it is not easy to combine IS with SG \citep{fu2017cpsg}; the same paper is, to our knowledge, the closest to this goal and will be compared with in Sec. \ref{subsec:bnn}.
EWSG can be viewed as a way to combine (adaptive) IS with SG for efficient sampling. It require no oracle about the gradient, nor any evaluation over the full data set. Instead, an inner-loop Metropolis chain maintains a random index that  approximates a state-dependent non-uniform distribution (i.e. the weights/importance).

\paragraph{Other Mini-batch MCMC Methods}
Besides SG-MCMC methods, there are also many non-gradient-based MCMC methods 
that use only a subset of data in each iteration so that the MCMC methods can scale to large data sets. For example, austerity MH \citep{korattikara2014austerity} formulates Metropolis-Hastings step as a statistical hypothesis testing problem and proposes to use only a subset of data to make statistically significant accept/reject decision. Using a subsampled unbiased estimator of the likelihood in a pseudo-marginal framework to accelerate the Metropolis-Hastings algorithm is proposed in \cite{bardenet2017markov}. A notable \emph{exact} MCMC method is FlyMC \cite{maclaurin2015firefly}, which introduces an auxiliary binary random variable for each datum and only the subset of data whose corresponding auxiliary binary indicator "light" up, are used in iteration. Some more recent advances on exact MCMC methods include \cite{zhang2019poisson, zhang2020asymptotically}. We also refer to \cite{bardenet2017markov} for an excellent review on subsampling MCMC methods.

\section{Underdamped Langevin: the continuous time backbone of a MCMC method} \label{background} 
Underdamped Langevin Dynamics (ULD) is given by the SDE
\begin{align}
\label{eq:uld}
    \begin{cases}
      d\bs{\bs{\theta}} &= \bs{r}dt\\
      d\bs{r} &= -(\nabla V(\bs{\bs{\theta}}) + \gamma \bs{r}) dt + \sigma d\bs{W}
    \end{cases}
    &
\end{align}
where $\bs{\bs{\theta}}, \bs{r} \in \mathbb{R}^d$ are state and momentum variables, $V$ is a potential energy function which in our context is, as originated from cost minimization or Bayesian inference over many data, the sum of many terms $V(\bs{\bs{\theta}}) = \sum_{i=1}^n V_i(\bs{\theta})$, $\gamma$ is a friction coefficient, $\sigma$ is intrinsic noise amplitude, and $\bs{W}$ is a standard $d$-dimensional Wiener process. Under mild assumptions on $V$, Langevin dynamics admits a unique invariant distribution $\pi(\bs{\theta}, \bs{r}) \sim \exp\left(-\frac{1}{T}(V(\bs{\theta}) + \frac{\|\bs{r}\|^2}{2})\right)$
and is in many cases geometric ergodic \citep{pavliotis2014stochastic}. $T$ is the temperature of system determined via the fluctuation dissipation theorem $\sigma^2 = 2 \gamma T$ \citep{kubo1966fluctuation}.

We consider ULD instead of the overdamped version mainly for two reasons: (i) one may think ULD is more complicated, and we'd like to show it is still easy to be paired with EWSG (EWSG can work for many MCMC methods; Appendix \ref{ewsg:sgld} has an overdamped version); (ii) it is believed that ULD has faster convergence than overdamped Langevin for instance in high-dimensions where (local) condition number is likely to be larger (e.g., \cite{cheng2017underdamped,cheng2018sharp,tao2020variational}). Like the overdamped version, numerical integrators for ULD with well captured statistical properties of the continuous process have been extensively investigated (e.g, \cite{roberts1996exponential,bou2010long}), and both the overdamped and underdamped integrators are friendly to derivations that will allow us to obtain explicit expressions of the non-uniform weights. 

\section{Method} \label{main}
\subsection{Motivation: An  Illustration of Non-optimality of Uniform Subsampling}
Uniform subsampling of gradients have long been the dominant way of stochastic gradient approximations mainly because it is intuitive, unbiased and easy to implement.

However, uniform gradient subsampling can introduce large noise, and is sub-optimal even in the family of unbiased stochastic gradient estimator, as the following Theorem \ref{thm:1} will show. One intuition is, consider for example cases where data size $n$ is larger than dimension $d$. In such cases, $\{\nabla V_i\}_{i=1,2,\cdots, n} \subset \mathbb{R}^d$ are linearly dependent and hence it is likely that there exist probability distributions $\{p_i\}_{i=1,2,\cdots,n}$ other than the uniform one such that the gradient estimate is unbiased, however with smaller variance because linearly dependent terms need not to be all used. This is a motivation for us to develop non-uniform subsampling schemes (weights may be $\bs{\theta}$ dependent), although we will not require $n>d$ later.

\begin{theorem}\label{thm:1}
Suppose given $\bs{\theta} \in \mathbb{R}^d$, the errors of SG approximation 
$
\bs{b}_i = n\nabla V_i(\bs{\theta}) - \nabla V(\bs{\theta}) , 1\le i \le n
$
are i.i.d. absolutely continuous random vectors with possibly-$\bs{\theta}$-dependent density $p(\cdot|\bs{\theta})$ and $n > d$. We call $\bs{p} \in \mathbb{R}^n$ a sparse vector if the number of non-zero entries in $\bs{p}$ is no greater than $d + 1$, i.e. $\|\theta\|_0 \le d+1$. Then with probability $1$,  the optimal probability distribution $\bs{p}^\star$ that is unbiased and minimizes the trace of the covariance of $n \nabla V_I(\bs{\theta})$, i.e. $\bs{p}^\star$ which solves the following, is a sparse vector. 
\begin{align} \label{eq:minimizetrace}
    \min_{\bs{p}}  \text{\rm Tr}(\mathbb{E}_{I \sim \bs{p}}[\bs{b}_I \bs{b}_I^T]) \quad 
    \text{s.t. } \mathbb{E}_{I \sim \bs{p}}[\bs{b}_I] = \bs{0},
\end{align}

\end{theorem}
Despite the sparsity of $\bs{p}^\star$, which seemingly suggests one only needs at most $d+1$ gradient terms per iteration when using SG methods, it is not practical because $\bs{p}^\star$ requires solving the linear programming problem \eqref{eq:minimizetrace} in Theorem \ref{thm:1}, for which an entire data pass is needed. Nevertheless, this result motivates us to seek alternatives to uniform SG. For example, the EWSG method we will develop will have reduced local variance with high probability, and at the same time remain efficiently implementable without having to use all data per parameter update; it can be biased though, but a global error analysis (Thm.\ref{thm:mse}) will show that trading bias for variance can still be worthy.

\subsection{Exponentially Weighted Stochastic Gradient}
\label{sec:method:EWSG}
MCMC methods are characterized by their transition kernels. In traditional SG-MCMC methods, uniform SG is used, which is independent of the intrinsic randomness of MCMC methods (e.g. diffusion in ULD), as a result, the transition kernel of SG-MCMC is quite different from that with full gradient. Therefore, it is natural to ask --- is it possible to couple the two originally independent randomness, so that the transition kernel of the SG-MCMC better matches that of the batch-gradient-MCMC, and the sampling accuracy is thus improved?

Here is one way to do so. Consider Euler-Maruyama (EM) discretization\footnote{EM is not the most accurate or robust discretization, see e.g., \citep{roberts1996exponential,bou2010long}, but since it may still be the most used method, demonstrations here will be based on EM. The same idea of EWSG can easily apply to most other discretizations such as GLA \citep{bou2010long}.} of Eq. \eqref{eq:uld}:
\begin{equation}\label{EM}
    \begin{cases}
      \bs{\theta}_{k+1} &= \bs{\theta}_k + \bs{r}_k h\\
      \bs{r}_{k+1} &= \bs{r}_k -(\nabla V(\bs{\theta}_k) + \gamma \bs{r}_k) h + \sigma \sqrt{h} \bs{\xi}_{k+1}
    \end{cases}
\end{equation}
where $h$ is step size and $\bs{\xi}_{k+1}$'s are i.i.d. $d$-dimensional standard Gaussian random variables. Denote the transition kernel of EM discretization with full gradient by $P^{EM}(\bs{\theta}_{k+1}, \bs{r}_{k+1} | \bs{\theta}_k, \bs{r}_k)$. 

Then consider a SG version: replace $\nabla V(\bs{\theta}_k)$ by a weighted SG $n \nabla V_{I_k}(\bs{\theta}_k)$, where $I_k$ is the index chosen to approximate full gradient and has p.m.f. $\mathbb{P}(I_k=i | \bs{\theta}_k, \bs{r}_k)=p_i$. Denote the new transition kernel by $\tilde{P}^{EM}(\bs{\theta}_{k+1}, \bs{r}_{k+1} | \bs{\theta}_k, \bs{r}_k)$.

It is not hard to see that
\begin{align}
    & P^{EM}(\bs{\theta}_{k+1}, \bs{r}_{k+1} | \bs{\theta}_k, \bs{r}_k) \nonumber\\
    =&  1_{\{\bs{\theta}_k + \bs{r}_k h\}}(\bs{\theta}_{k+1}) \, \frac{1}{Z} \exp\left( -\frac{\| \bs{r}_{k+1} - \bs{r}_k + (\nabla V(\bs{\theta}_k) + \gamma \bs{r}_k)h\|^2}{2\sigma^2 h}\right) \nonumber\\
    =& 1_{\{\bs{\theta}_k + \bs{r}_k h\}}(\bs{\theta}_{k+1}) \, \frac{1}{Z} \exp\left( -\frac{\| \bs{x} + \sum_{i=1}^n \bs{a}_i\|^2}{2}\right) \nonumber 
\end{align}
and
\[
    \tilde{P}^{EM}(\bs{\theta}_{k+1}, \bs{r}_{k+1} | \bs{\theta}_k, \bs{r}_k) 
    = 1_{\{\bs{\theta}_k + \bs{r}_k h\}}(\bs{\theta}_{k+1}) \, \frac{1}{\tilde{Z}} \sum_{j=1}^n p_i  \exp\left( -\frac{\|\bs{x} + n\bs{a}_i\|^2}{2}\right),
\]
where $Z$ and $\tilde{Z}$ are normalization constants, $ \bs{x} \triangleq \frac{\bs{r}_{k+1} - \bs{r}_k + h \gamma \bs{r}_k}{\sigma \sqrt{h}}$ and  $\bs{a}_i \triangleq \frac{\sqrt{h} \nabla V_i (\bs{\theta}_k)}{\sigma}$. From these two expressions, one can see that if we could choose
\begin{equation}
    p_i \propto \exp\left( -\frac{\| \bs{x} + \sum_{i=1}^n \bs{a}_i\|^2}{2} + \frac{\|\bs{x} + n\bs{a}_i\|^2}{2} \right),
    \label{eq:transition-kernel}
\end{equation}
we would have $P^{EM}(\bs{\theta}_{k+1}, \bs{r}_{k+1} | \bs{\theta}_k, \bs{r}_k) = \tilde{P}^{EM}(\bs{\theta}_{k+1}, \bs{r}_{k+1} | \bs{\theta}_k, \bs{r}_k)$ and  be able to recover the transition kernel of full gradient with that of stochastic gradient. However, Eq.\eqref{eq:transition-kernel} is only formal and infeasible, because $\bs{x}$ is dependent on the future state variable $\bs{r}_{k+1}$ which we do not know. Therefore, to obtain a practically implementable algorithm, we will fix $\bs{x}$ as a hyper-parameter and hope that the approximation is good enough so that we still have $P^{EM}(\bs{\theta}_{k+1}, \bs{r}_{k+1} | \bs{\theta}_k, \bs{r}_k) \approx \tilde{P}^{EM}(\bs{\theta}_{k+1}, \bs{r}_{k+1} | \bs{\theta}_k, \bs{r}_k)$.


We refer to the choice of $p_i$ in eq.\ref{eq:transition-kernel} Exponentially Weighted Stochastic Gradient (\textbf{EWSG}). Unlike Thm.\ref{thm:1}, EWSG does not require $n > d$ to work. Note the idea of designing non-uniform weights of SG-MCMC to match the transition kernel of full gradient can be suitably applied to a wide class of gradient-based MCMC methods; for example, Sec. \ref{ewsg:sgld} shows how EWSG can be applied to Langevin Monte Carlo (overdamped Langevin), and Sec. \ref{variance reduction} shows how it can be combined with VR. Therefore, EWSG complements a wide range of SG-MCMC methods.

Since the weight choice of EWSG is motivated by approximating the transition kernel of a full-gradient MCMC method, we anticipate EWSG to be statistically more accurate than a uniformly-subsampled stochastic gradient estimator. As a special but commonly interested accuracy measure, the smaller variance of EWSG is shown with high probability\footnote{`With high probability' but not almost surely because Theorem \ref{thm:local} in fact suits a class of weights, which includes but is not limited to EWSG.}:
\begin{theorem}\label{thm:local}
Assume $\{\nabla V_i(\bs{\theta})\}_{i=1,2,\cdots,n}$ are i.i.d random vectors and  $|\nabla V_i(\bs{\theta})|\leq R$ for some constant $R$ almost surely. Denote the uniform distribution over $[n]$ by $\bs{p}^U$, the exponentially weighted distribution by $\bs{p}^E$, and let $\Delta = \Tr[\text{\rm cov}_{I\sim\bs{p}^E}[n\nabla V_I(\bs{\theta})|\bs{\theta}] - \text{\rm cov}_{I\sim\bs{p}^U}[n\nabla V_I(\bs{\theta})|\bs{\theta}]]$. If $\bs{x}=\mathcal{O}(\sqrt{h})$, we have 
$\mathbb{E} [\Delta] < 0$, and $\exists C>0$ independent of $n$ or $h$ such that $\forall \epsilon > 0$,
\[
\mathbb{P}(|\Delta - \mathbb{E}[\Delta]| \ge \epsilon) \le 2\exp\left(-\frac{\epsilon^2}{n C h^2}\right).
\]
\end{theorem}
It is not surprising that less non-intrinsic local variance correlates with better global statistical accuracy, which will be made explicit and rigorous in the next subsection.

\subsection{Non-asymptotic Error Bound} \label{sec:theory}

We now establish a non-asymptotic global sampling error bound (in mean square distance between arbitrary test observables) of SG underdamped Langevin algorithms (the bound applies to both EWSG and other methods e.g., SGHMC). The full proof is deferred to the Appendix \ref{sec:globalErrorProof}, but the main tool we will be using is the Poisson equation machinery \citep{mattingly2010convergence,vollmer2016exploration,chen2015convergence}. A brief overview is the following:

Let $\bs{X} = \left( \begin{matrix} \bs{\theta} \\ \bs{r} \end{matrix} \right)$. The generator $\mathcal{L}$ of diffusion process \eqref{eq:uld} is 
\begin{align*}
    \mathcal{L} (f(\bs{X}_t)) =& \lim_{h \to 0} \frac{\mathbb{E}[f(\bs{X}_{t+h})] - \mathbb{E}[f(\bs{X}_t)]}{h} \\
    =& \bs{r}^T \nabla_{\bs{\theta}}f  - (\gamma \bs{r} + \nabla V(\bs{\theta}))^T \nabla_{\bs{r}}f + \gamma \Delta_{\bs{r}}f.
\end{align*}
 
Given a test function $\phi(\bs{x})$, its posterior average is $\bar{\phi} = \int \phi(\bs{x}) \pi(\bs{x}) d\bs{x}$, approximated by its time average of samples $\widehat{\phi}_K = \frac{1}{K} \sum_{k=1}^K \phi(\bs{X}^E_k)$, where $\bs{X}_k^{E}$ is the sample path given by EM integrator. Then the Poisson equation $\mathcal{L}\psi = \phi - \bar{\phi}$ can be a useful tool for the weak convergence analysis of SG-MCMC.
The solution $\psi$ characterizes the difference between $\phi$ and its posterior average $\bar{\phi}$.

Our main theoretical result is the following:
\begin{theorem} \label{thm:mse}
Assume $\mathbb{E}[\|\nabla V_i(\bs{\theta}^E_k)\|^l] < M_1, \mathbb{E}[\|\bs{r}^E_k\|^l] < M_2, \forall l = 1, 2, \cdots, 12, \forall i=1,2,\cdots,n$ and $\forall k \ge 0$. Assume the solution to the Poisson equation, $\psi$, exists, and its derivatives up to 3rd-order are uniformly bounded $ \|D^l \psi\|_\infty < M_3, l=0,1,2,3$. Then there exist constants $C_1, C_2, C_3>0$ depending on $M_1, M_2, M_3$, such that 
\begin{equation} \label{eq:mse}
    \mathbb{E}\big( \widehat{\phi}_K - \bar{\phi} \big)^2 \leq
    C_1 \frac{1}{T} + C_2 \frac{h}{T} \frac{\sum_{k=0}^{K-1} \mathbb{E}[\Tr[\mbox{\rm cov}(n\nabla V_{I_k} | \mathcal{F}_k)]]}{K} + C_3 h^2
\end{equation}
where $T = Kh$ is the corresponding time in the underlying continuous dynamics, $I_k$ is the index of the datum used to estimate the gradient at $k$-th iteration, and $\mbox{cov}(n\nabla V_{I_k} | \mathcal{F}_k)$ is the covariance of stochastic gradient at $k$-th iteration conditioned on the current sigma algebra $\mathcal{F}_k$ in the filtration.
\end{theorem}
\begin{remark} (\emph{interpreting the three terms in the bound})
    Unlike a typical VR method which aims at finding unbiased gradient estimator with reduced variance, EWSG aims at bringing the entire density closer to that of a batch-gradient MCMC. As a consequence, its practical implementation may correspond to SG that has reduced variance but a small bias too. Eq.\eqref{eq:mse} quantifies this bias-variance trade-off. How the extrinsic local variance and bias contribute to the global error is respectively reflected in the 2nd and 3rd terms, although the 3rd term also contains a contribution from the numerical discretization error. With or without bias, the 3rd term remains $\mathcal{O}(h^2)$ because of this discretization error. However, for moderate $T$, the 2nd term is generally larger than the 3rd due to its lower order in $h$, which means reducing local variance can improve sampling accuracy even if at the cost of introducing a small bias. Since EWSG has a smaller local variance than uniform SG (Thm.\ref{thm:local}, as a special case of improved overall statistical accuracy), its global performance is also favorable. The 1st term is for the convergence of the continuous process (eq.\ref{eq:uld} in this case).
\end{remark}
\begin{remark} (\emph{innovation and relation with the literature})
    Thm.\ref{thm:mse}, to the best of our knowledge, is the first that incorporates the effects of both local bias and local variance of a SG approximation (previous SOTA bounds are only for unbiased SG). It still works when restricting to unbiased SG, and in this case our bound reduces to SOTA \cite{vollmer2016exploration, chen2015convergence}. Some more facts include: \cite{mattingly2010convergence}, being the seminal work from which we adapt our proof, only discussed the batch gradient case, whereas our theory has additional (non-uniform) SG. \cite{vollmer2016exploration,chen2015convergence} studied the effect of SG, but the SG considered there did not use state-dependent weights, which would destroy several martingales used in their proofs. Unlike in \cite{mattingly2010convergence} but like in \cite{vollmer2016exploration,chen2015convergence}, our state space is not the compact torus but $\mathbb{R}^d$. Also, the time average $\widehat{\phi}_K$, to which our results apply, is a commonly used estimator, particularly when using a long time trajectory of Markov chain for sampling. However, if one is interested in an alternative of using an ensemble for sampling, techniques in \cite{cheng2017underdamped, dalalyan2017user} might be useful to further bound difference between the law of $\bs{X}_k$ and the target distribution.
\end{remark}

\subsection{Practical Implementation}
In EWSG, the probability of each gradient term is
$
    p_i = \widehat{Z}^{-1} \exp\left\{-\frac{\|\bs{x} + \sum_{j=1}^n \bs{a}_j\|^2}{2} + \frac{\|\bs{x} + n \bs{a}_i\|^2}{2}\right\}
$.
Although the term $\|\bs{x} + \sum_{j=1}^n \bs{a}_j\|^2/2$ depends on the full data set, it is shared by all $p_i$'s and can be absorbed into the normalization constant $\hat{Z}^{-1}$ (we still included it explicitly due to the needs in proofs); unique to each $p_i$ is only the term $\|\bs{x} + n \bs{a}_i\|^2/2$. This motivates us to run a Metropolis-Hastings chain over the possible indices $i\in \{1,2\,\cdots,n\}$: at each inner-loop step, a proposal of index $j$ is uniformly drawn, and then accepted with probability $P(i \to j ) =$
\begin{equation} \label{eq:transition}
    \min\left\{1,  \exp\left(\frac{\|\bs{x} + n \bs{a}_j\|^2}{2} - \frac{\|\bs{x} + n \bs{a}_i\|^2}{2}\right)\right\};
\end{equation}
if accepted, the current index $i$ is replaced by $j$. When the chain converges, the index will follow the distribution given by $p_i$. The advantage is, we avoid passing through the entire data sets to compute each $p_i$, but the index will still approximately sample from the non-uniform distribution.

In practice, we often perform only $M=1$ step of the Metropolis index chain per integration step, especially if $h$ is not too large. The rationale is, when $h$ is small, the outer iteration evolves slower than the index chain, and as $\theta$ does not change much in, say, $N$ outer steps, effectively $N\times M$ inner steps take place on almost the same index chain, which makes the index r.v. equilibrate better. Regarding the larger $h$ case (where the efficacy of local variance reduction via non-uniform subsampling is more pronounced; see e.g., Thm. \ref{thm:mse}), $M=1$ may no longer be optimal, but improved sampling with large $h$ and $M=1$ is still clearly observed in various experiments (Sec.  \ref{experiments}). 

Another hyper-parameter is $\bs{x}$ (see earlier discussion in Sec.\ref{sec:method:EWSG}). Our heuristic recommendation is $\bs{x}=\frac{\sqrt{h} \gamma \bs{r}_k}{\sigma}$.
The rationale is, as long as $r_{k+1}-r_k$'s density is maximized at 0 (which will be the case at least for large $k$ as $r_k$ will converge to a Gaussian), this choice of $\bs{x}$ is a maximum likelihood estimator. This approximation appeared to be a good one in all our experiments with medium $h$ and $M=1$. 


Sec. \ref{subsec:simple_gaussian} further investigates hyperparameter selection empirically and shows that approximations due to $M$ and $\bs{x}$ is not detrimental to our non-asymptotic theory in Sec. \ref{sec:theory}.

Practical EWSG is summarized in Algorithm \ref{alg:EWSG}. For simplicity of notation, we restrict the description to mini batch size $b = 1$, but an extension to $b > 1$ is straightforward. See Sec.   \ref{ewsg:minibatch} in appendix. Practical EWSG has reduced variance but does not completely eliminate the extrinsic noise created by SG due to its approximations. A small bias was also created by these approximations, but its effect is dominated by the variance effect (see Sec.  \ref{sec:theory}). In practice, if needed, one can combine EWSG with other VR technique to further improve accuracy. Appendix \ref{variance reduction} describes how EWSG can be combined with SVRG.

\setlength{\textfloatsep}{5pt}
\begin{algorithm}[tb]
   \caption{EWSG}\label{alg:EWSG}
\begin{algorithmic}
   \STATE {\bfseries Input:} \{the number of data terms $n$, gradient functions $V_i(\cdot), i=1,2,\cdots,n$, step size $h$, the number of data passes $K$, index chain length $M$, friction and noise coefficients $\gamma$ and  $\sigma$\}
   \STATE Initialize $\bs{\theta}_0, \bs{r}_0$ (arbitrarily, or use an informed guess)
   \FOR{$k = 0,1,\cdots, \lceil \frac{K n}{ M + 1}  \rceil$}
   \STATE $i \gets$ uniformly sampled from ${1,\cdots,n}$, \, compute and store $n\nabla V_i(\bs{\theta}_k)$
   \STATE $I \gets i$
        \FOR{$m = 1,2,\cdots, M$}
         \STATE $j \gets$ uniformly sampled from ${1,\cdots,n}$, \, compute and store $n\nabla V_j(\bs{\theta}_k)$
         \STATE $I \gets j$ with probability in Equation \ref{eq:transition}
        \ENDFOR
        \STATE Evaluate $\tilde{V}(\bs{\theta}_k) = n V_I(\bs{\theta}_k)$
        \STATE Update $(\bs{\theta}_{k+1}, \bs{r}_{k+1}) \gets (\bs{\theta}_k, \bs{r}_k)$ via one step of Euler-Maruyama integration using $\tilde{V}(\bs{\theta}_k)$
   \ENDFOR
\end{algorithmic}
\end{algorithm}

\section{Experiments} \label{experiments}
In this section, the proposed EWSG algorithm will be compared with SGHMC \citep{chen2014stochastic}, SGLD  \citep{welling2011bayesian}, as well as several more recent popular approaches, including  FlyMC \citep{maclaurin2015firefly}, pSGLD \citep{li2016preconditioned},  CP-SGHMC \citep{fu2017cpsg} (a method closest to the goal of applying IS idea to SG-based sampling) and SVRG-LD \citep{dubey2016variance} (overdamped Langevin improved by VR). Sec.   \ref{subsec:simple_gaussian} is a detailed empirical study of EWSG on simple models, with comparison and implication of two important hyper-parameters $M$ and $\bs{x}$, and verification of the non-asymptotic theory (Theorem \ref{thm:mse}). Sec.   \ref{subsec:blr} demonstrates  EWSG for Bayesian logistic regression on a large-scale data set. Sec.   \ref{subsec:bnn} is a Bayesian Neural Network (BNN) example. It serves only as a high-dimensional, multi-modal test case, and we do not intend to compare Bayesian and non-Bayesian neural nets. As FlyMC requires a tight lower bound of likelihood, known for only a few cases, it will only be compared against in Sec.    \ref{subsec:blr} where such a bound is obtainable. CP-SGHMC requires heavy tuning on the number of clusters which differs across data sets/algorithms, so it will only be included in the BNN example, for which the authors empirically found a good hyper parameter for MNIST \citep{fu2017cpsg}.
SVRG-LD is only compared to in Sec.   \ref{subsec:simple_gaussian}, because SG-MCMC methods can converge within only one data pass in Sec.   \ref{subsec:blr}, rendering control-variate based VR technique inapplicable, and it was suggested that VR leads to poor results for deep models (e.g., Sec.  \ref{subsec:bnn}) \citep{defazio2019ineffectiveness}.

For fair comparison, all algorithms use constant step sizes and are allowed fixed computation budget, i.e., for $L$ data passes, all algorithms can only call gradient function $nL$ times.  All experiments are conducted on a machine with a 2.20GHz Intel(R) Xeon(R) E5-2630 v4 CPU and an Nvidia GeForce GTX 1080 GPU. If not otherwise mentioned, $\sigma = \sqrt{2\gamma}$ so only $\gamma$ needs specification, the length of the index chain is set $M=1$ for EWSG and the default values of two hyper-parameters required in pSGLD are set $\lambda=10^{-5}$ and $\alpha=0.99$, as suggested in \cite{li2016preconditioned}.

\subsection{Gaussian Examples} \label{subsec:simple_gaussian}
Consider sampling from a simple 2D Gaussian whose potential function is
$$
V(\bs{\theta}) = \sum_{i=1}^n V_i(\bs{\theta}) = \sum_{i=1}^n \frac{1}{2}\|\bs{\theta} - \bs{c}_i\|^2.
$$
We set $n = 50$ and randomize $\bs{c}_i$ from a two-dimensional standard normal $\mathcal{N}(\bs{0}, I_2)$.
Due to the simplicity of $V(\bs{\theta})$, we can write the target density analytically 
and will use KL divergence
$
  \mathrm{KL}(p \| q) = \int p(\bs{\theta}) \log \frac{p(\bs{\theta})}{q(\bs{\theta})}d\bs{\theta}
$
to measure the difference between the target distribution and generated samples. 

\begin{figure*}[!t]
\centering
    \begin{subfigure}{0.30\textwidth}
		\centering
		\includegraphics[width=\textwidth]{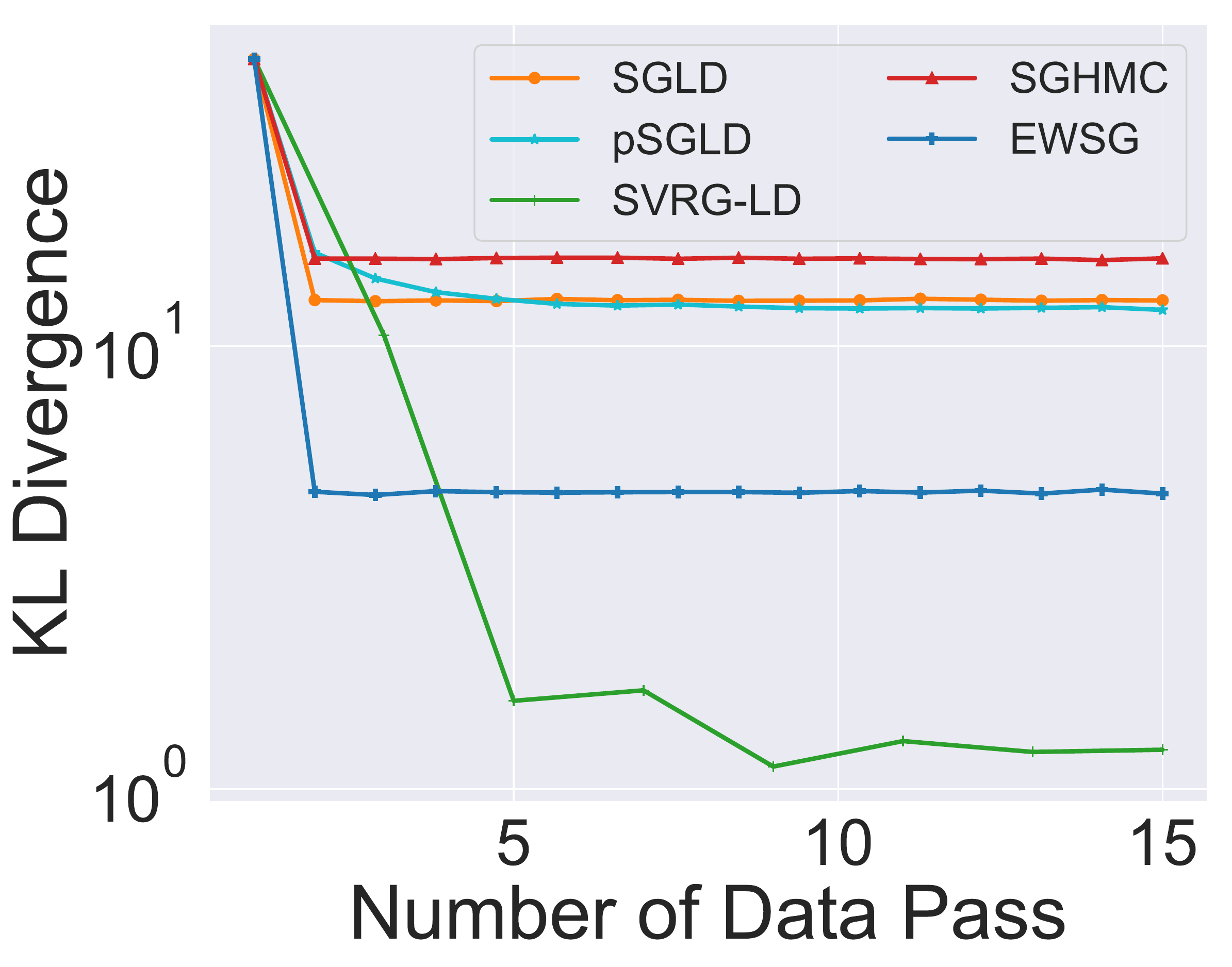}
		\caption{Sample quality in KL} \label{fig:gaussian_kl}	
	\end{subfigure}
    \begin{subfigure}{0.30\textwidth}
		\centering
		\includegraphics[width=\textwidth]{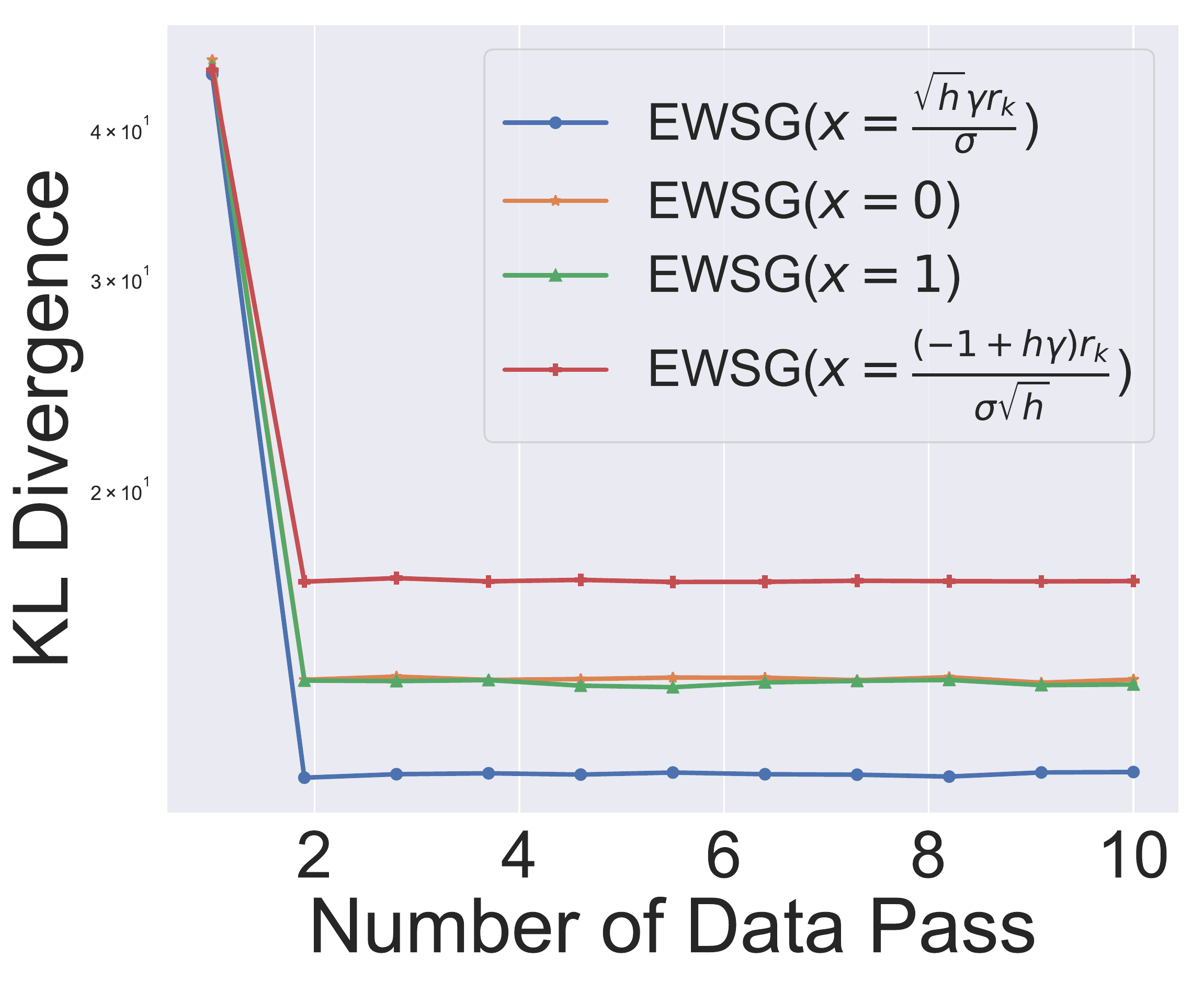}
		\caption{Performance for various $\bs{x}$} \label{fig:gaussian_x}
	\end{subfigure}
    \begin{subfigure}{0.30\textwidth}
		\centering
		\includegraphics[width=\textwidth]{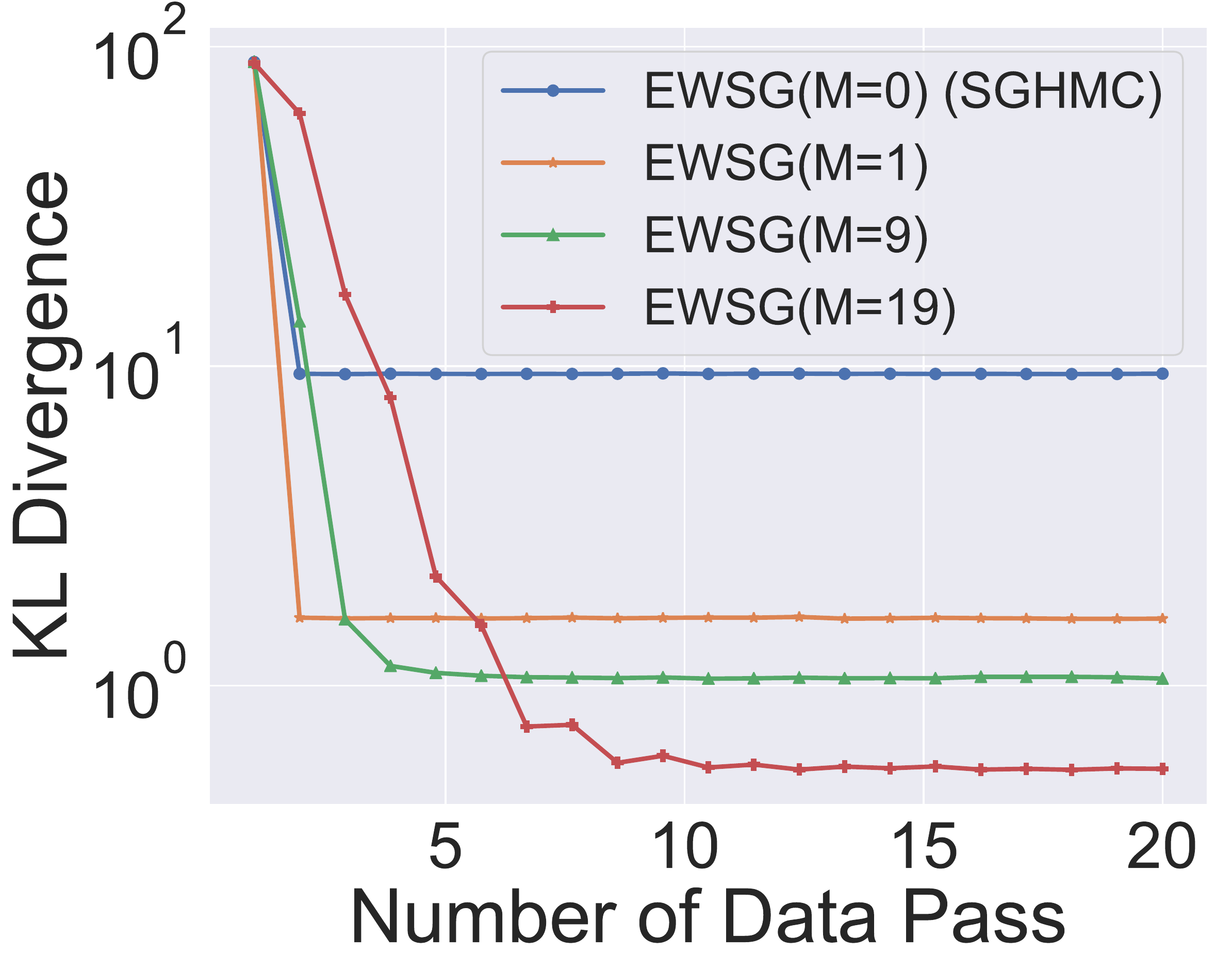}
		\caption{Performance for various $M$} \label{fig:gaussain_M_KL}
	\end{subfigure}
    \begin{subfigure}{0.30\textwidth}
        \centering
        \includegraphics[width=\textwidth]{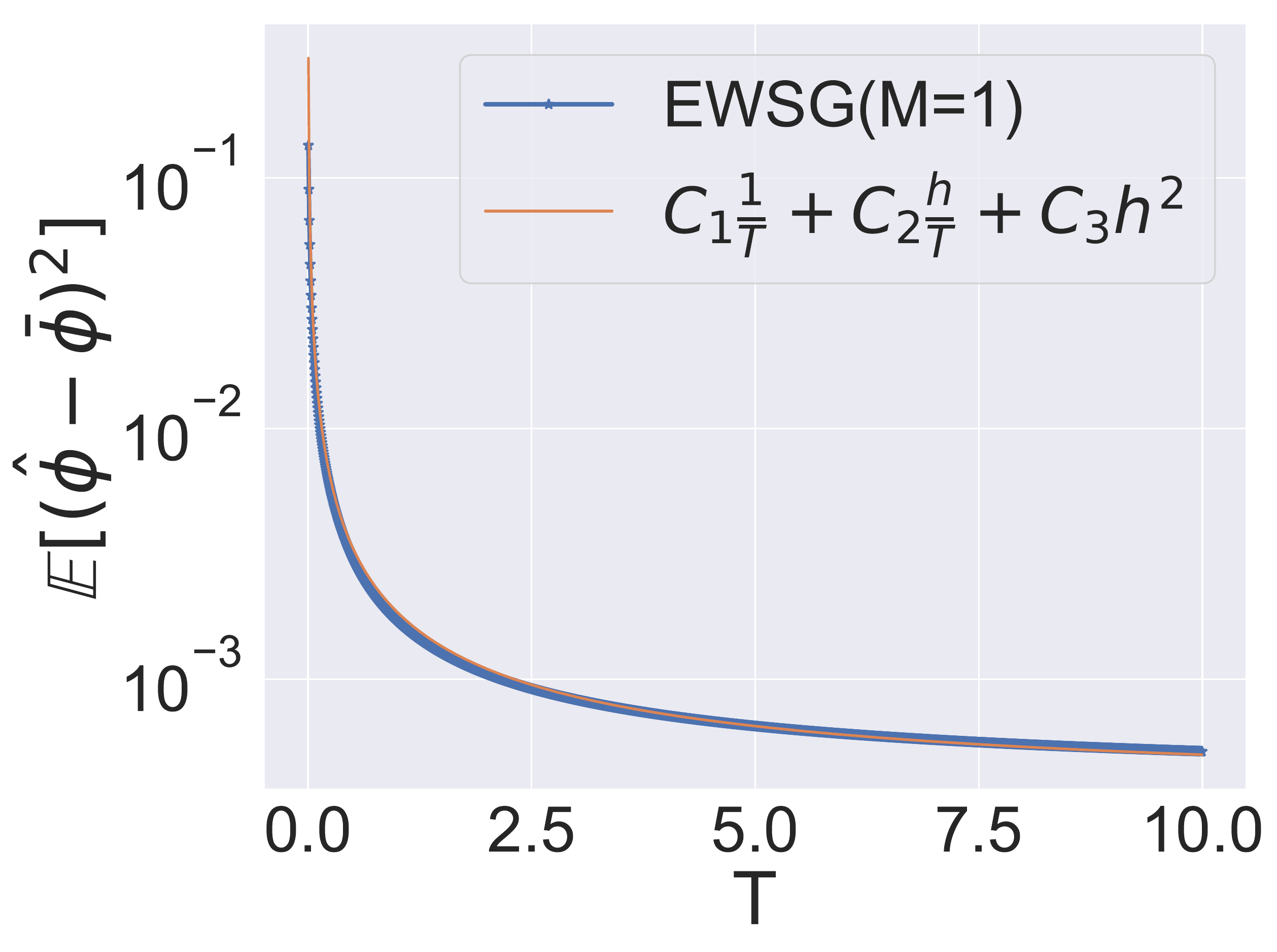}
        \caption{MSE against time $T$ (1\textsuperscript{st} and 2\textsuperscript{nd} terms in Eq. \eqref{eq:mse})} \label{fig:rebuttal_2_T}
    \end{subfigure}
    \begin{subfigure}{0.30\textwidth}
        \centering
        \includegraphics[width=\textwidth]{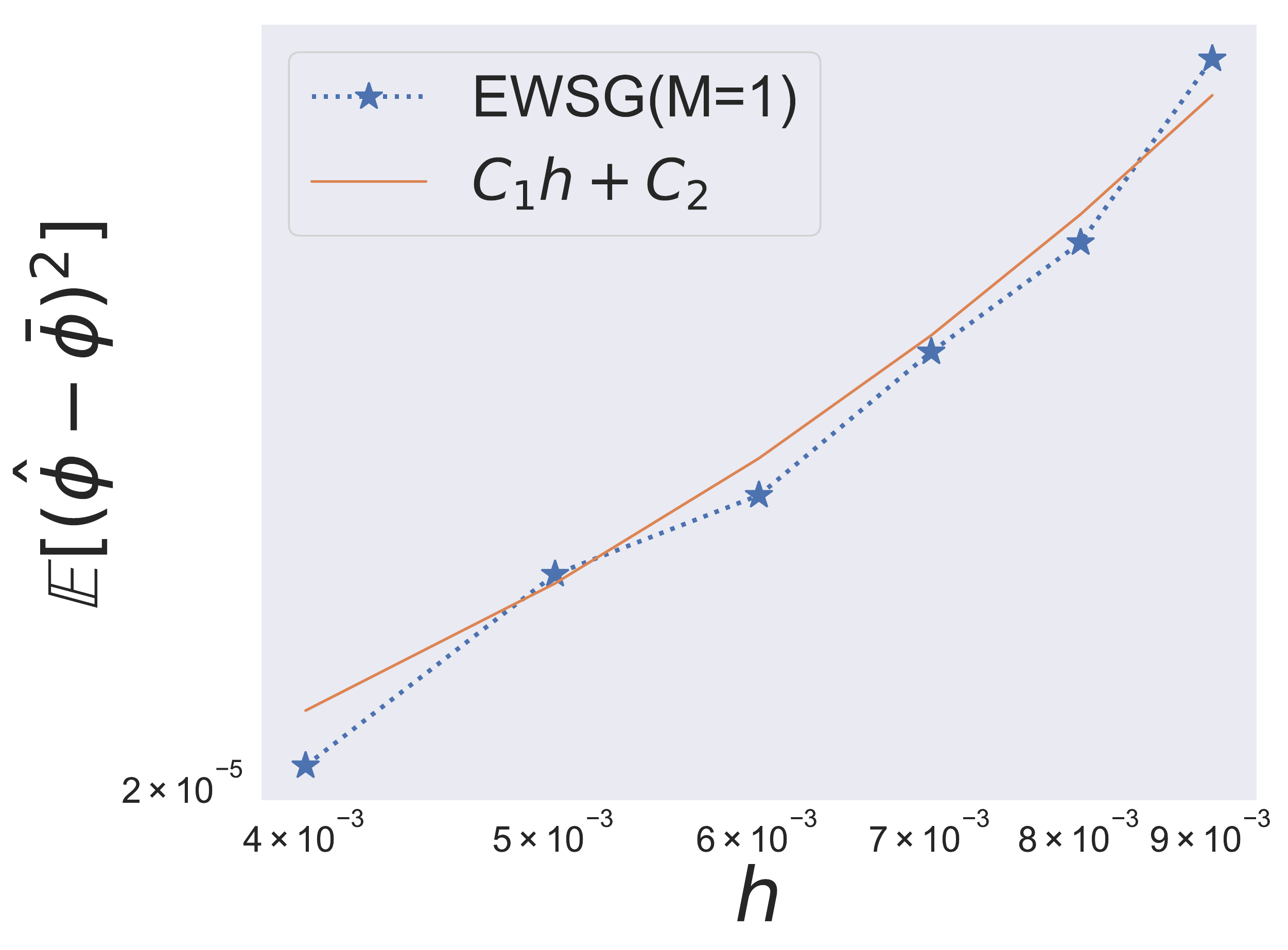}
        \caption{MSE against step size $h$ with fixed finite
        $T$ 
        (2\textsuperscript{nd} term in Eq. \eqref{eq:mse})} \label{fig:rebuttal_2_second_term_h}
    \end{subfigure}
    \begin{subfigure}{0.30\textwidth}
        \centering
        \includegraphics[width=\textwidth]{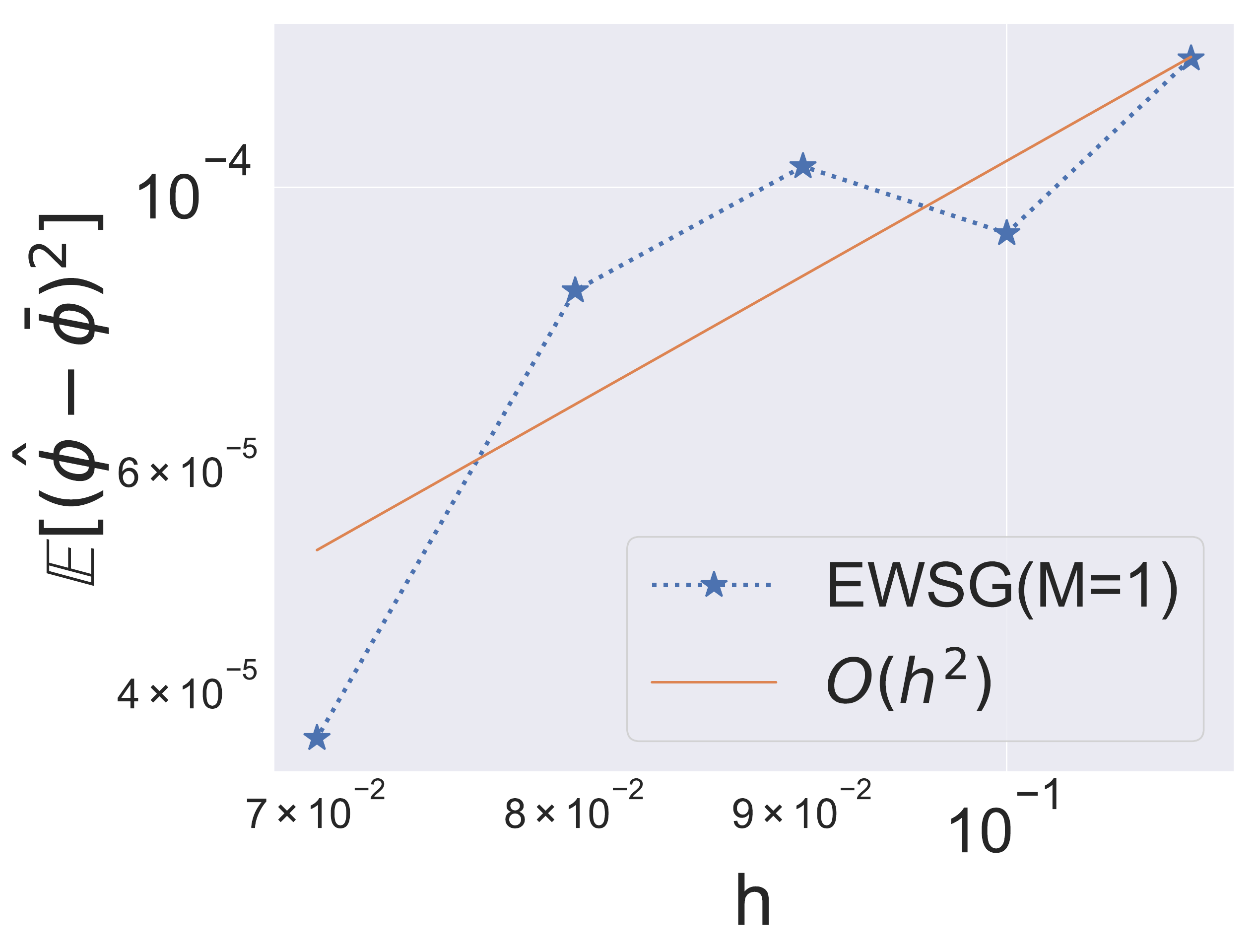}
        \caption{MSE against step size $h$ with $T\approx\infty$ (3\textsuperscript{rd} term in Eq. \eqref{eq:mse})} \label{fig:rebuttal_2_h}
    \end{subfigure}
	
    \caption{Sampling from Gaussian target}
\end{figure*}

For each algorithm, we generate 10,000 independent realizations for empirical estimation. All algorithms are run for 30 data passes with minibatch size of 1. Step size is tuned from $5 \times \{10^{-1}, 10^{-2}, 10^{-3}, 10^{-4}\}$ and $5\times10^{-3}$ is chosen for SGLD and pSGLD, $5\times10^{-2}$ for SGHMC and EWSG and $5\times10^{-4}$ for SVRG-LD. SGHMC and EWSG use $\gamma=10$. Results are shown in Fig. \ref{fig:gaussian_kl} and EWSG outperforms SGHMC, SGLD and pSGLD in terms of accuracy. Note SVRG-LD has the best accuracy\footnote{For Gaussians, mean and variance completely determine the distribution, so appropriately reduced variance leads to great accuracy for the entire distribution.} but
the slowest convergence, and that is why EWSG is a useful alternative to VR: its light-weight suits situations with limited computational resources better.



Figure \ref{fig:gaussian_x} shows the performance of several possible choices of the hyper-parameter $\bs{x}$, including the recommended option $\bs{x} = \sqrt{h}\gamma \bs{r}_k / \sigma$, and $\bs{x}=\bs{0}$, $\bs{x}=\bs{1}$, $\bs{x} = (-1+h\gamma)\bs{r}_k / \sigma\sqrt{h}$ (which corresponds to $\bs{r}_{k+1} = \bs{0}$). Step size $h = 7 \times 10^{-2}$ is used for this experiment. The recommended option performs better than the others.

Another important hyper-parameter in EWSG is $M$. As the length of index chain $M$ increases, the subsampling distribution approaches that given by Eq.\ref{eq:transition-kernel}. 
Considering that larger $M$ means more gradient evaluations per step\footnote{in each iteration of the outer MCMC loop, EWSG consumes $M+1$ data points, and hence in a fair comparison with fixed computation budget (e.g. $E$ total gradient calls), EWSG runs $\frac{E}{M + 1}$ iterations which is decreasing in $M$.}, there could be some $M$ value that achieves the best balance between speed and accuracy. Fig.\ref{fig:gaussain_M_KL} shows a fair comparison of four values of $M=0, 1, 9, 19$, and the recommended $M=1$ case converges as fast as SGHMC (when $M=0$, EWSG does not run the Metropolis-Hastings index chain and hence degenerates to SGHMC) but improves its accuracy. It is also clear that as $M$ increases, sampling accuracy gets improved.

As approximations are used in Algorithm \ref{alg:EWSG}, it is natural to ask if results of Thm. $\ref{thm:mse}$ still hold. We  empirically investigate this question (using $M=1$ and variance as the test function $\phi$). Eq.\ref{eq:mse} in Thm.\ref{thm:mse} is a nonasymptotic error bound consisting of three parts, namely an $\mathcal{O}(\frac{1}{T})$ term corresponding to the convergence at the continuous limit, an $\mathcal{O}(h/T)$ term coming from the SG variance, and an $\mathcal{O}(h^2)$ term due to bias and numerical error. Fig.\ref{fig:rebuttal_2_T} plots the mean squared error (MSE) against time $T = Kh$ to confirm the 1st term. Fig.\ref{fig:rebuttal_2_second_term_h} plots the MSE against $h$ with fixed $T$ in the small $h$ regime (so that the 3rd term is negligible when compared to the 2nd) to confirm that the 2nd term scales like $\mathcal{O}(h)$. 
For the 3rd term in Eq. \eqref{eq:mse}, we run sufficiently many iterations to ensure all chains are well-mixed, and Fig.\ref{fig:rebuttal_2_h} confirms the final MSE to scale like $\mathcal{O}(h^2)$ even for large $h$ (as the 2nd term vanishes due to $T\to\infty$). In this sense, despite the  approximations introduced by the practical implementation, the performance of Algorithm \ref{alg:EWSG} is still approximated by Thm. \ref{thm:mse}, even when $M=1$. Thm. \ref{thm:mse} can thus guide the choices of $h$ and $T$ in practice.

\begin{table}[!t]
    \centering
    {
    \small
    \begin{tabular}{cccccc}
        \Xhline{3\arrayrulewidth}
        Method  & SGLD &  pSGLD & SGHMC & EWSG & FlyMC   \\
        Accuracy(\%) & 75.283 $\pm$ 0.016 & 75.126 $\pm$ 0.020 & 75.268 $\pm$ 0.017 & \textbf{75.306} $\pm$ \textbf{0.016} & 75.199 $\pm$ 0.080 \\
        Log Likelihood & -0.525 $\pm$ 0.000 & -0.526 $\pm$ 0.000 & -0.525 $\pm$ 0.000 & \textbf{-0.523} $\pm$ \textbf{0.000} & \textbf{-0.523} $\pm$ \textbf{0.000}\\
        Wall Time (s) & \textbf{3.085 $\pm$ 0.283} &  4.312 $\pm$ 0.359 & 3.145 $\pm$ 0.307 & 3.755 $\pm$ 0.387 & 291.295 $\pm$ 56.368 \\ 
        \Xhline{3\arrayrulewidth}
    \end{tabular}   
    }
    \caption{Accuracy, log likelihood and wall time of various algorithms on test data after one data pass (mean $\pm$ std).}\label{tab:blr}   
\end{table}

\subsection{Bayesian Logistic Regression (BLR)} \label{subsec:blr}
Consider Bayesian logistic regression for the binary classification problem.
The probabilistic model for predicting a label $y_k$ given a feature vector $x_k$ is $p(y_k=1|\bs{x}_k, \bs{\theta}) = 1 / (1 + \exp(-\bs{\theta}^T\bs{x}_k))$. We set a Gaussian prior with zero mean and covariance $\Sigma = 10I_d$ for $\bs{\theta}$, and hence the potential function of the posterior distribution of $\bs\theta$ is 
$$
    V(\bs{\theta}) = 5\|\bs{\theta}\|^2 -\sum_{i=1}^n y_i \log p(y_i=1|\bs{x}_i, \bs{\theta}) + (1-y_i) \log \left(1-p(y_i=1|\bs{x}_i, \bs{\theta}) \right).
$$ 
We conduct our experiments on Covertype data set\footnote{https://archive.ics.uci.edu/ml/datasets/covertype}, which contains $n=581,012$ data points and 54 features (which is the dimension of $\bs\theta$). Given the large size of this data set, SG is needed to scale up MCMC methods. We use 80\% of data for training and the rest 20\% for testing.

The FlyMC algorithm\footnote{https://github.com/HIPS/firefly-monte-carlo/tree/master/flymc} uses a lower bound derived in \cite{maclaurin2015firefly} for likelihood function. For underdamped Langevin based algorithms, we set friction coefficient $\gamma=50$. After tuning, we set the step size as $\{1, 3, 0.02, 5, 5\} \times 10^{-3}$ for SGULD, EWSG, SGLD, pSGLD and FlyMC. All algorithms are run for one data pass, with minibatch size of 50 (for FlyMC, it means 50 data are sampled in each iteration to switch state). 100 independent samples are drawn from each algorithm to estimate statistics. To further smooth out noise, all experiments are repeated 1000 times with different seeds.



Results are in Fig. \ref{fig:blr_acc} and \ref{fig:blr_llh} and Table \ref{tab:blr}. EWSG outperforms others, except for log likelihood being comparable to FlyMC, which is an \textit{exact} MCMC method. The wall time consumed by EWSG is only slightly more than that of SGLD and SGHMC, but less than pSGLD and orders-of-magnitude less than FlyMC.

\begin{figure}
\begin{minipage}{0.52\textwidth}
    \begin{subfigure}{0.48\textwidth}
		\centering
		\includegraphics[width=\textwidth]{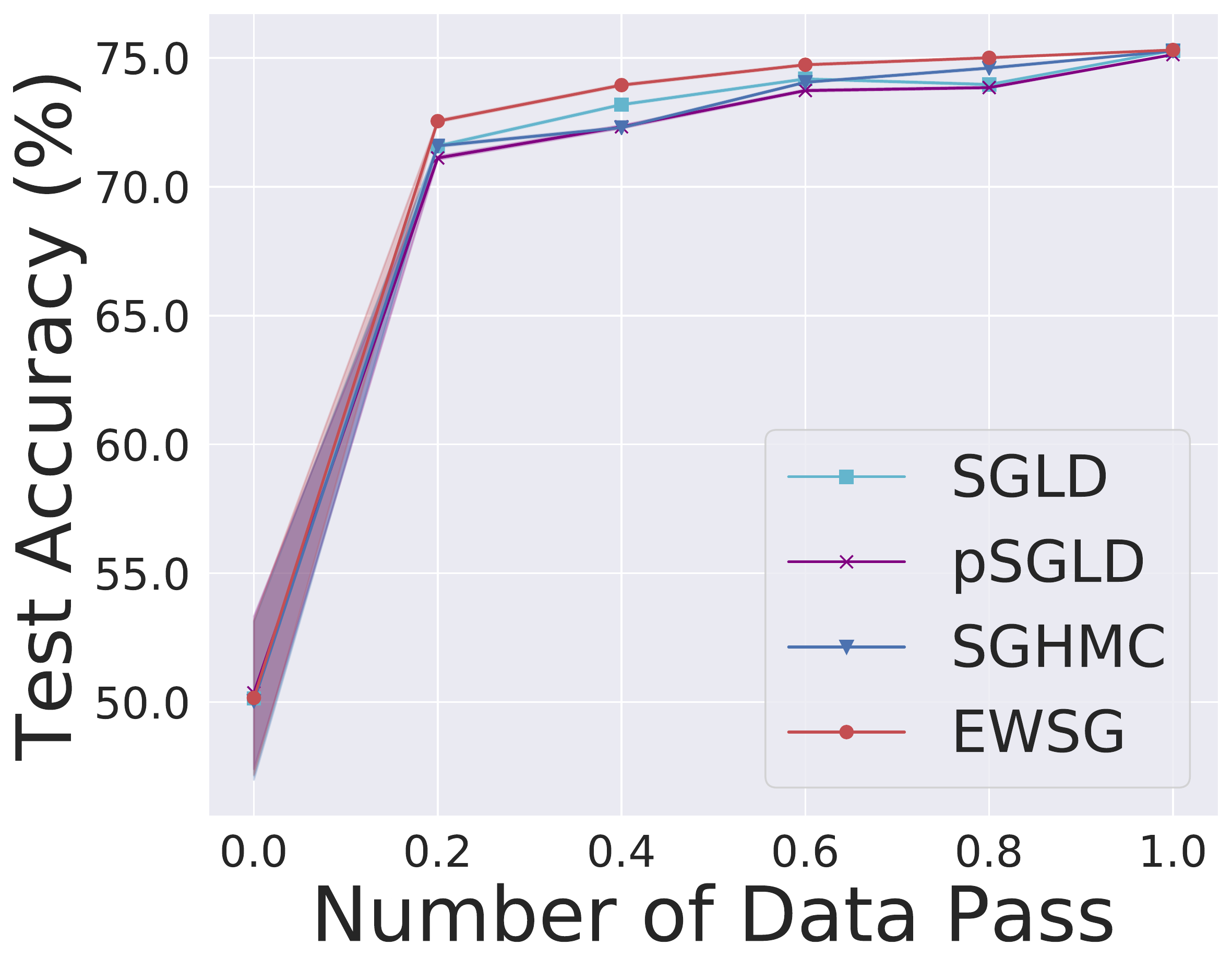}
		\caption{Test Accuracy} \label{fig:blr_acc}
	\end{subfigure}
    \begin{subfigure}{0.48\textwidth}
		\centering
		\includegraphics[width=\textwidth]{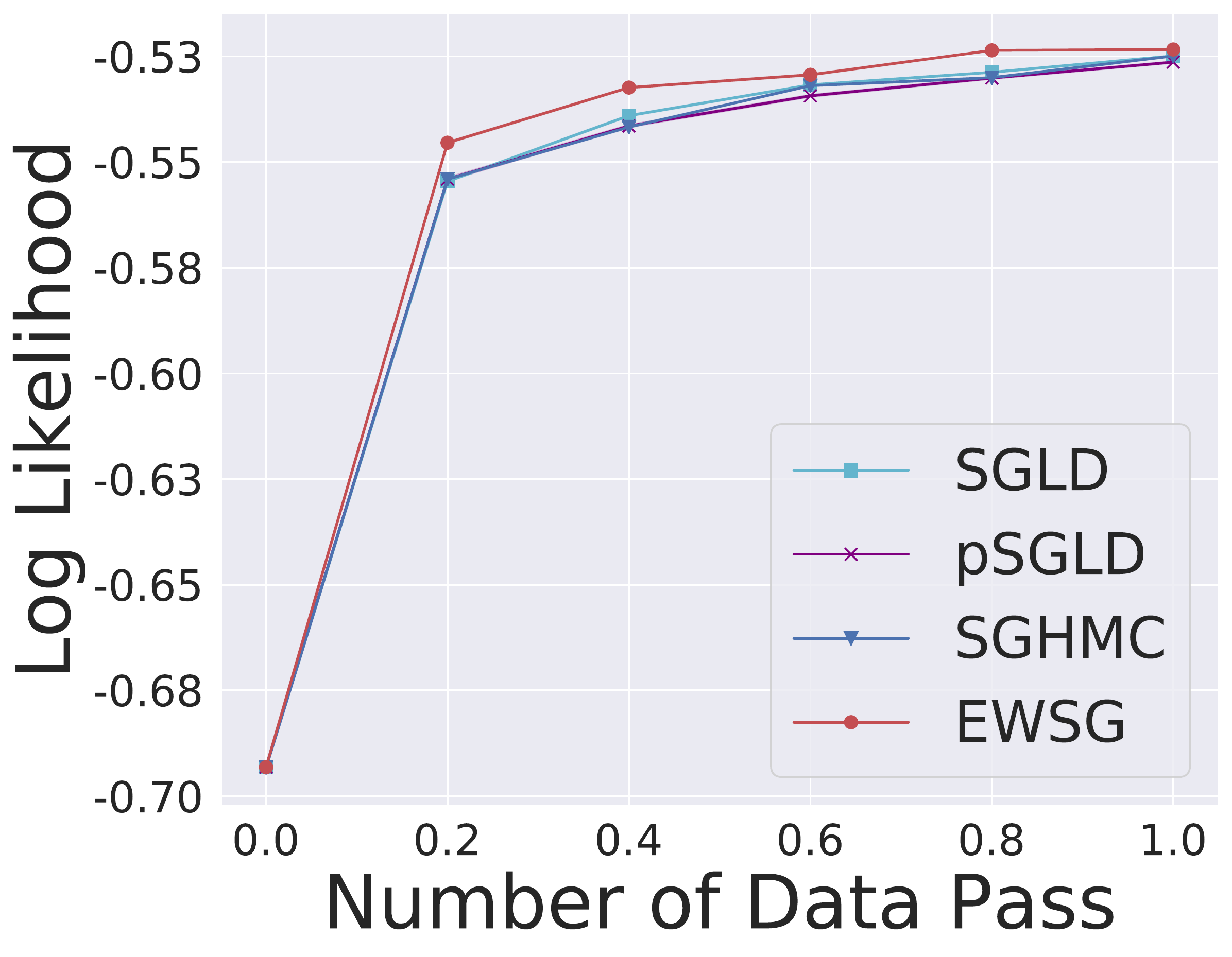}
		\caption{Test Log Likelihood} \label{fig:blr_llh}
	\end{subfigure}
	\caption{BLR learning curve}
\end{minipage}\hfill
\begin{minipage}{0.48\textwidth}
    \begin{subfigure}{0.48\textwidth}
		\centering
		\includegraphics[width=\textwidth]{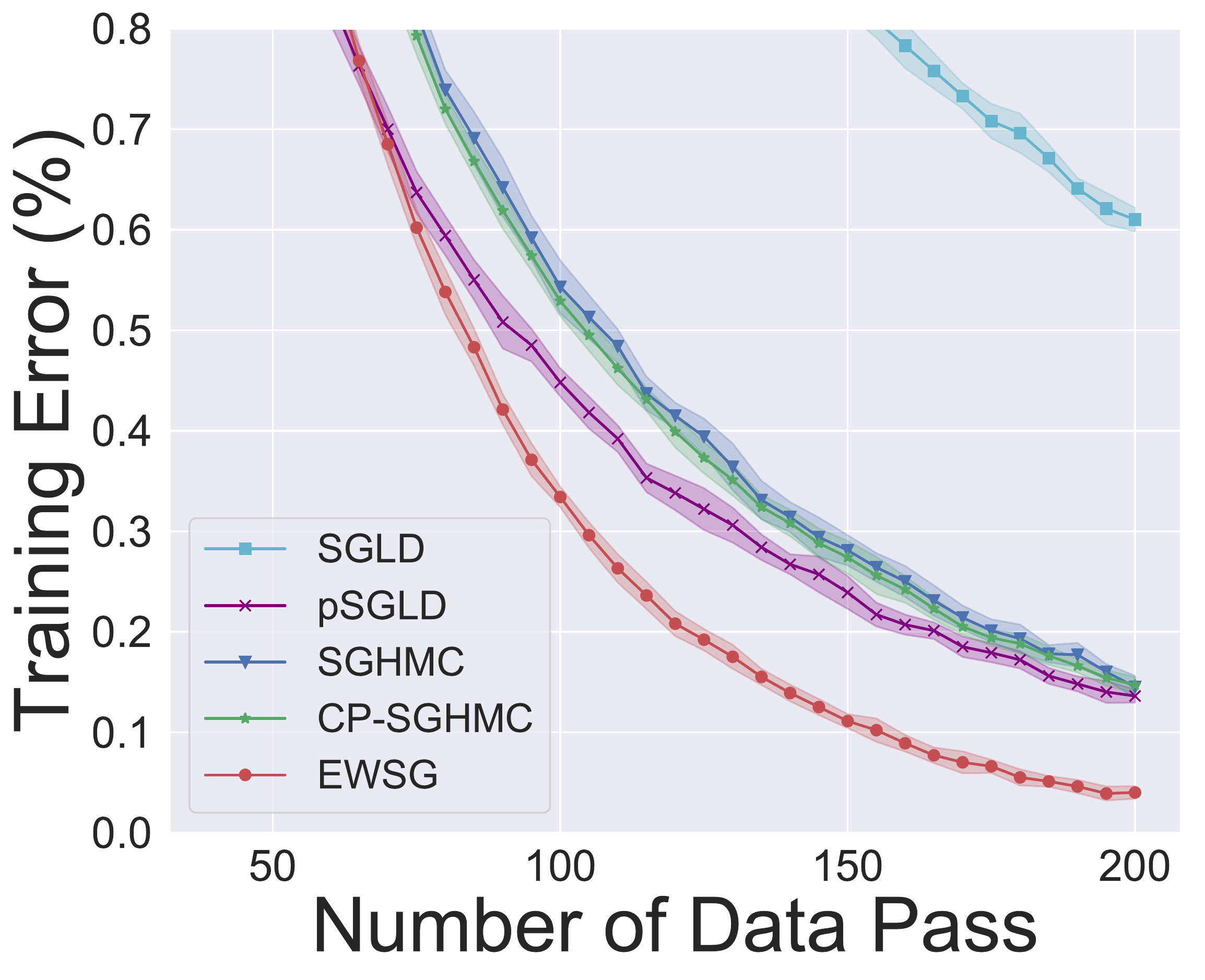}
		\caption{MLP architecture} \label{fig:mlp}
	\end{subfigure}
    \begin{subfigure}{0.48\textwidth}
		\centering
		\includegraphics[width=\textwidth]{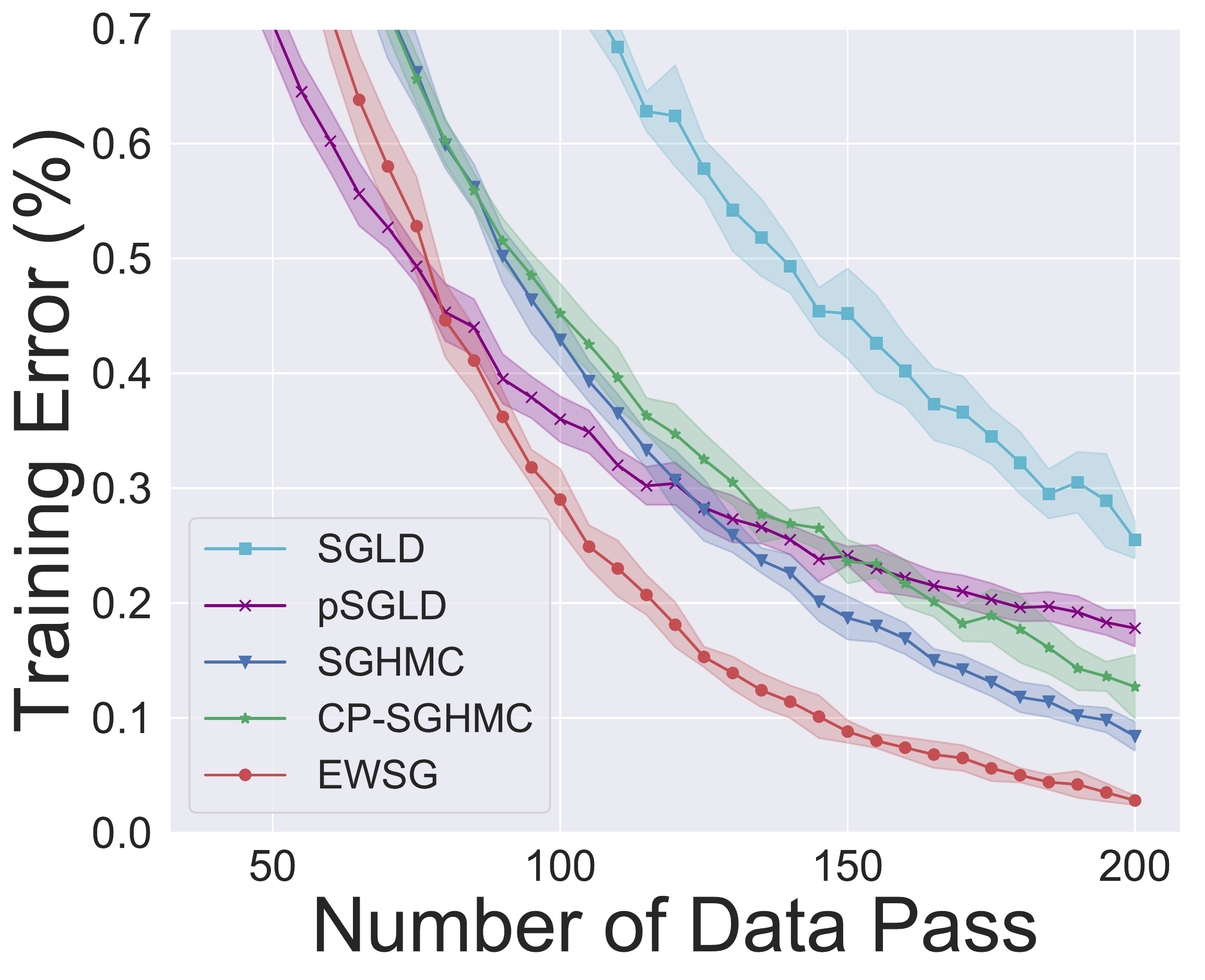}
		\caption{CNN architecture} \label{fig:cnn}
	\end{subfigure}
    \caption{BNN learning curve. Shade: one standard deviation.} \label{fig:bnn_learning_curve}
\end{minipage}
\end{figure}	

\subsection{Bayesian Neural Network (BNN)}\label{subsec:bnn}
Bayesian neural network is a compelling model for deep learning \citep{wilson2020case}. Here two popular architectures of BNN are experimented -- multilayer perceptron (MLP) and convolutional neural nets (CNN). In MLP, a hidden layer with 100 neurons followed by a softmax layer is used. In CNN, we use standard network configuration with 2 convolutional layers followed by 2 fully connected layers \citep{jarrett2009best}. Both convolutional layers use $5 \times 5$ convolution kernel with 32 and 64 channels, $2\times2$ max pooling layers follow immediately after convolutional layer. The last two fully-connected layers each has 200 neurons. We set the standard normal as prior for all weights and bias.

We test algorithms on the MNIST data set, consisting of 60,000 training data and 10,000 test data, each datum is a $28 \times 28$ gray-scale image with one of the ten possible labels (digits $0\sim9$). For ULD based algorithms , we set friction coefficient $\gamma=0.1$ in MLP and $\gamma=1.0$ in CNN. In MLP, the step sizes are set $h= \{4, 2, 2\} \times 10^{-3}$ for EWSG, SGHMC and CP-SGHMC, and $h = \{0.001, 1\}\times 10^{-4}$ for SGLD and pSGLD, via grid search. For CP-SGHMC ,
(clustering-based preprocessing is conducted \citep{fu2017cpsg} before SGHMC)
we use K-means with 10 clusters to preprocess the data set. In CNN, the step sizes are set $h= \{4, 2, 2\} \times 10^{-3}$ for EWSG, SGHMC and CP-SGHMC, and $h = \{0.02, 8\} \times 10^{-6}$ for SGLD and pSGLD, via grid search. All algorithms use minibatch size of 100 and are run for 200 epoches. For each algorithm, we generate 100 independent samples to make posterior prediction. To smooth out noise and obtain more significant results, we repeat experiments 10 times with different seeds.

The learning curve of training error is shown in Fig. \ref{fig:mlp} and \ref{fig:cnn}. EWSG consistently improves over its uniform counterpart (i.e., SGHMC) and CP-SGHMC (an approximate IS SG-MCMC). Moreover, EWSG also outperforms two standard benchmarks SGLD and pSGLD. The improvement over baseline on MNIST data set is comparable to some of the early works \citep{chen2014stochastic, li2016preconditioned}. 

Note: in the MLP setup, the model has $d > 78,400$ parameters whereas there are $n=60,000$ data points, which shows EWSG does not require $n > d$ to work and can still outperform its uniform counterpart in the overparametrized regime (Thm.\ref{thm:1} demonstrates the underparametrized case only because the sparsity result is easy to understand, but EWSG doesn't only work for underparameterized models).  

\section*{Acknowledgment}
We thank Yian Ma and Xuefeng Gao for insightful discussions. MT is grateful for the partial support by NSF DMS-1847802 and ECCS-1936776. This work was mainly conducted when HZ was a professor at Georgia Institute of Technology.

\newpage
\bibliography{reference}
\bibliographystyle{plainnat}

\appendix
\clearpage
\onecolumn
\section{Proof of Theorem \ref{thm:1}}
\begin{proof}
Denote the set of all $n$-dimensional probability vectors by $\Sigma^n$, the set of sparse probability vectors by $\mathcal{S}$, and the set of non-sparse (dense) probability vectors by $\mathcal{D} = \Sigma^n \setminus \mathcal{S}$.
Denote $B=[\bs{b}_1, \cdots, \bs{b}_n]$, then the optimization problem can be written as 
\begin{align*}
    &\min \sum_{i=1}^n p_i \|\bs{b}_i\|^2 \\
    &\text{s.t.}  
    \begin{cases}
      \displaystyle
      B\bs{p} = \bs{0}\\
      \bs{p}^T\bs{1}_n = 1 \\
      p_i \ge 0, i=1,2,\cdots,n
    \end{cases}
\end{align*}
Note that the feasible region is always non-empty (take $\bs{p}$ to be a uniform distribution) and is also closed and bounded, hence this linear programming is always solvable. Denote the set of all minimizers by $\mathcal{M}$. Note that $\mathcal{M}$ depends on $\bs{b}_1, \cdots, \bs{b}_n$ and is in this sense random.

The Lagrange function is 
\[ L(\bs{p}, \bs{\lambda}, \mu, \bs{\omega}) = \bs{p}^T\bs{s} - \bs{\lambda}^TB\bs{p} - \mu(\bs{p}^T\bs{1}_n) - \bs{\omega}^T\bs{p}\]
where $ \bs{s}= [\|\bs{b}_1\|^2, \|\bs{b}_2\|^2,\cdots, \|\bs{b}_n\|^2]^T$ and $\bs{\lambda}, \mu, \bs{\omega}$ are dual variables. The optimality condition reads as 
\begin{equation*}
\frac{\partial L}{\partial \bs{p}}  = \bs{s} - B^T\bs{\lambda} - \mu \bs{1}_n - \bs{\omega} = \bs{0}
\end{equation*}
Dual feasibilty and complementary slackness require
\begin{align*}
    & \omega_i \leq 0, i=1,2,\cdots,n \\
    & \bs{\omega}^T\bs{p} = 0
\end{align*}

Consider the probability of the event \{a dense probability vector can solve the above minimization problem\}, i.e., $\mathbb{P}(\mathcal{M} \cap \mathcal{D} \neq \emptyset)$. It is upper bounded by
\[
    \mathbb{P}(\mathcal{M} \cap \mathcal{D} \neq \emptyset) \le \mathbb{P}(\bs{p} \in \mathcal{D} \mbox{ and } \bs{p} \mbox{ solves KKT condition})
\]

Since $\bs{p} \in \mathcal{D}$, complementary slackness implies that at least $d + 2$ entries in $\bs{\omega}$ are zero. Denote the indices of these entries by $\mathcal{J}$. For every $j \in \mathcal{J}$, by optimality condition, we have $s_j - \bs{\lambda}^T \bs{b}_j - \mu = 0$, i.e.,
\[
    \|\bs{b}_j\|^2 - \bs{\lambda}^T\bs{b}_j - \mu = 0
\]
Take the first $d + 1$ indices in $\mathcal{J}$, and note a geometric fact that $d + 1$ points in a $d$-dimensional space must be on the surface of a hypersphere of at most $d-1$ dimension, which we denote by $\mathcal{S}=S^{q-1}+\bs{x}$ for some vector $\bs{x}$ and integer $q\leq d$. Because $b_i$'s distribution is absolutely continuous, we have 
\begin{align*}
    & \mathbb{P}(\bs{p} \in \mathcal{D} \mbox{ and } \bs{p} \mbox{ solves KKT condition}) \\
    \le& \mathbb{P}(\bs{p} \in \mathcal{D} \mbox{ and } \bs{b}_j \in S , \forall j \in \mathcal{J}) \\
    \le& \mathbb{P}(\bs{b}_j \in S  , \forall j \in \mathcal{J}) \\
    =& \mathbb{P}(\bs{b}_{j_k} \in S  , k = d+2, \cdots, |\mathcal{J}|) \\
    =& \prod_{k=d+2}^{|\mathcal{J}|} \mathbb{P}(\bs{b}_{j_k} \in S) \qquad \mbox{(independence)} \\
    =& 0 \qquad \mbox{(absolute continuous)}
\end{align*}
Hence $\mathbb{P}(\mathcal{M} \cap \mathcal{D} \neq \emptyset) = 0$ and 
\begin{align*}
    1 &= \mathbb{P}(\mathcal{M} \neq \emptyset) \\
      &= \mathbb{P} ((\mathcal{M} \cap \mathcal{S}) \cup (\mathcal{M} \cap \mathcal{D}) \neq \emptyset) \\
      &\le \mathbb{P} (\mathcal{M} \cap \mathcal{S} \neq \emptyset) + \mathbb{P}  (\mathcal{M} \cap \mathcal{D} \neq \emptyset) \\
      &= \mathbb{P} (\mathcal{M} \cap \mathcal{S} \neq \emptyset)
\end{align*}
Therefore we have 
\[
    \mathbb{P} (\mathcal{M} \cap \mathcal{S} \neq \emptyset) = 1
\]
\end{proof}

\section{Proof of Theorem \ref{thm:local}}
\begin{proof}
Let $\bs{b}_i = n \nabla V_i$ and assume $\|\bs{b}_i\|_2 \leq R$ for some constant $R$. Denote $B = [\bs{b}_1, \bs{b}_2, \cdots, \bs{b}_n]$. For any probability distribution $\bs{p}$ over $\{1,\cdots,n\}$, we have 
\begin{align*}
& \text{cov}_{I \sim \bs{p}}[\bs{b}_I|\bs{b}_1,\cdots,\bs{b}_n] \\
=& \sum_{i=1}^n p_i \bs{b}_i\bs{b}_i^T - \left(\sum_{i=1}^n p_i \bs{b}_i\right)\left(\sum_{i=1}^n p_i \bs{b}_i\right)^T \\
=& \sum_{i=1}^n p_i \bs{b}_i\bs{b}_i^T \sum_{i=1}^n p_i - \left(\sum_{i=1}^n p_i \bs{b}_i\right)\left(\sum_{i=1}^n p_i \bs{b}_i\right)^T\\
=& \sum_{i < j} (\bs{b}_i-\bs{b}_j)(\bs{b}_i-\bs{b}_j)^T p_i p_j
\end{align*}
Therefore we let
\begin{align*}
f(B) := &  \Tr\left[\sum_{i < j} (\bs{b}_i-\bs{b}_j)(\bs{b}_i-\bs{b}_j)^T p_i p_j - \sum_{i < j} (\bs{b}_i-\bs{b}_j)(\bs{b}_i-\bs{b}_j)^T \frac{1}{n^2}\right] \\
= & \sum_{i < j} \|\bs{b}_i-\bs{b}_j\|^2 p_i p_j - \sum_{i < j} \|\bs{b}_i-\bs{b}_j\|^2 \frac{1}{n^2} \qquad (\Tr[AB] = \Tr[BA])
\end{align*}
and use it to compare the trace of covariance matrix of uniform- and nonuniform- subsamplings.

First of all,
\begin{align*}
&\mathbb{E}[f(B)] \\
=& \mathbb{E}[\|\bs{b}_i - \bs{b}_j\|^2] \sum_{i<j}\left(p_ip_j-\frac{1}{n^2}\right) \\
=& \mathbb{E}[\|\bs{b}_i - \bs{b}_j\|^2] \left(\sum_{i<j}p_ip_j - \frac{n-1}{2n}\right) \\
=& \mathbb{E}[\|\bs{b}_i - \bs{b}_j\|^2] \left(\frac{1 - \sum_{i=1}^n p_i^2}{2} - \frac{n-1}{2n}\right) \\
\le& \mathbb{E}[\|\bs{b}_i - \bs{b}_j\|^2] \left(\frac{1 - \frac{1}{n}}{2} - \frac{n-1}{2n}\right)\\
=& 0
\end{align*}
where the inequality is due to Cauchy-Schwarz and it is a strict inequality unless all $p_i$'s are equal,
which means uniform subsampling on average has larger variablity than a non-uniform scheme measured by the trace of covariance matrix.

Moreover, concentration inequality can help show $f(B)$ is negative with high probability if $h$ is small. To this end, plug $\bs{x} = \mathcal{O}(\sqrt{h})$ in and rewrite 
\[ \displaystyle p_i = \frac{1}{Z} \exp\left\{Fh\left[\frac{\|\bs{y} + \frac{1}{n}\sum_{i=1}^n \bs{b}_i\|^2}{2} - \frac{\|\bs{y} + \bs{b}_i\|^2}{2}\right]\right\}\]
where $\bs{y} = \frac{\sigma}{\sqrt{h}} \bs{x} = \mathcal{O}(1)$, $F = -\frac{1}{\sigma^2}$ and $Z$ is the normalization constant. Denote the unnormalized probability by \[\widetilde{p}_i = \exp\left\{Fh\left[\frac{\|\bs{y} + \frac{1}{n}\sum_{i=1}^n \bs{b}_i\|^2}{2} - \frac{\|\bs{y} + \bs{b}_i\|^2}{2}\right]\right\}\] and we have
\begin{align*}
f(B) &= \frac{1}{2}\sum_{i=1}^n\sum_{j=1}^n \|\bs{b}_i - \bs{b}_j\|^2 \left(p_ip_j-\frac{1}{n^2}\right)\\
          &= \frac{1}{2}\sum_{i=1}^n\sum_{j=1}^n \|\bs{b}_i - \bs{b}_j\|^2 \frac{\widetilde{p}_i\widetilde{p}_j}{[\sum_{k=1}^n\widetilde{p}_k]^2} - \frac{1}{2}\sum_{i=1}^n\sum_{j=1}^n \|\bs{b}_i - \bs{b}_j\|^2\frac{1}{n^2}
\end{align*}

To prove concentration results, it is useful to estimate
\[C_i = \sup_{\substack{\bs{b}_1,\cdots, \bs{b}_{n} \in B(\bs{0}, R) \\ \widehat{\bs{b}}_i \in B(\bs{0}, R)}} |f(\bs{b}_1, \cdots, \bs{b}_i, \cdots, \bs{b}_n) \]\[- f(\bs{b}_1, \cdots, \widehat{\bs{b}}_i, \cdots, \bs{b}_n)| \]
where $B(\bs{0}, R)$ is a ball centered at origin with radius $R$ in $\mathbb{R}^d$.

Due to the mean value theorem, we have $C_i \le 2R \sup |\frac{\partial f}{\partial \bs{b}_i}|$. By symmetry, it suffices to compute $\sup |\frac{\partial f}{\partial \bs{b}_1}|$ to upper bound $C_1$. Note that 
\[
\frac{\partial \widetilde{p}_j}{\partial \bs{b}_1} = 2 \widetilde{p}_j F h [\frac{1}{n}(\bs{y} + \frac{1}{n}\sum_{i=1}^n \bs{b}_i) - (\bs{y} + \bs{b}_j)\delta_{1j}] = \mathcal{O}(h) \widetilde{p}_j
\]
where $\delta_{1j}$ is the Kronecker delta function. Thus
\begin{align*}
\frac{\partial f}{\partial \bs{b}_1} &= \sum_{j=1}(\bs{b}_1 - \bs{b}_j)\frac{\widetilde{p}_1\widetilde{p}_j}{[\sum_{k=1}^n\widetilde{p}_k]^2} - \sum_{j=1}^n (\bs{b}_1 - \bs{b}_j)\frac{1}{n^2}  + \sum_{i,j=1}^n\|\bs{b}_1 - \bs{b}_j\|^2\frac{ \mathcal{O}(h) \widetilde{p}_i \widetilde{p}_j}{[\sum_{k=1}^n\widetilde{p}_k]^2} \\
&- 2 \sum_{i,j=1}^n\|\bs{b}_1 - \bs{b}_j\|^2\frac{\widetilde{p}_i \widetilde{p}_j}{\left[\sum_{k=1}^n\widetilde{p}_k\right]^3} \sum_{k=1}^n \widetilde{p}_k \mathcal{O}(h)\\
&= \widetilde{p}_1 \sum_{j=1}^n (\bs{b}_1 - \bs{b}_j) \frac{\widetilde{p}_j}{[\sum_{k=1}^n\widetilde{p}_k]^2} - \sum_{j=1}^n (\bs{b}_1 - \bs{b}_j) \frac{1}{n^2} + \frac{\mathcal{O}(n^2)\mathcal{O}(h)}{\mathcal{O}(n^2)} + \frac{\mathcal{O}(n^2)}{\mathcal{O}(n^3)}\mathcal{O}(n)\mathcal{O}(h)\\
&= \mathcal{O}(\frac{h}{n}) + \mathcal{O}(h) + \mathcal{O}(h)\\
&= \mathcal{O}(h)
\end{align*}
where $\mathcal{O}(\frac{h}{n})$ in the 2nd last equation comes from the difference of the first two terms in the 3rd last equation. This estimation shows that $ C_i \le 2R \mathcal{O}(h) = \mathcal{O}(h)$.

Therefore, by McDiarmid's inequality, we conclude for any $\epsilon > 0$,
\[ 
\mathbb{P}(|f - \mathbb{E}[f]| > \epsilon) \le 2\exp\left(\frac{-2\epsilon^2}{\sum_{i=1}^n C_i^2}\right)
= 2\exp\left(\frac{-2\epsilon^2}{n\mathcal{O}(h^2)}\right).
\]
Any choice of $h(n)=o(n^{-1/2})$ will render this probability asymptotically vanishing as $n$ grows, which means that $f$ will be negative with high probability, which is equivalent to reduced variance per step.
\end{proof}

\section{Proof of Theorem \ref{thm:mse}}
\label{sec:globalErrorProof}
\begin{proof}
We rewrite the generator of underdamped Langevin with full gradient as
\[
    \mathcal{L} f(\bs{X}) = \bs{F}(\bs{X})^T \begin{bmatrix} \nabla_{\bs{\theta}} f(\bs{X}) \\ \nabla_{\bs{r}} f(\bs{X})\end{bmatrix} + \frac{1}{2} A:\nabla\nabla f(\bs{X})
\]
where 
\[
    \bs{F}(\bs{X}) = \begin{bmatrix} \bs{r} \\ -\gamma \bs{r} - \nabla V(\bs{\theta}) \end{bmatrix}, \quad A = GG^T \mbox{and}\quad  G =\begin{bmatrix} O_{d\times d} & O_{d\times d}\\
    O_{d\times d} & \sqrt{2\gamma} I_{d \times d} \end{bmatrix}
\]

Rewrite the discretized underdamped Langevin with stochastic gradient in variable $\bs{X}$, i.e.,
\[
    \bs{X}^E_{k+1} - \bs{X}^E_{k} = h\bs{F}_k(\bs{X}^E_{k})  + \sqrt{h} G_k \bs{\eta}_{k+1}
\]
where 
\[
    \bs{F}_{k}(\bs{X}) = \begin{bmatrix} \bs{r} \\ -\gamma \bs{r} - n\nabla V_{I_k} (\bs{\theta}) \end{bmatrix}, \qquad G_{k} = G = \begin{bmatrix} O_{d\times d} & O_{d\times d} \\
    O_{d\times d} & \sqrt{2\gamma} I_{d \times d} \end{bmatrix},
\]
and $\bs{\eta}_{k+1}$ is a $2d$ dimensional standard Gaussian random vector. Note that this representation include both SGHMC and EWSG, for SGHMC $I_k$ follows uniform distribution and for EWSG, $I_k$ follows the MCMC-approximated exponentially weighted distribution.

Denote the generator associated with stochastic gradient underdamped Langevin at the $k$-th iteration by 
\[
    \mathcal{L}_k f(\bs{X}) = \bs{F}_k(\bs{X})^T \begin{bmatrix} \nabla_{\bs{\theta}} f(\bs{X}) \\ \nabla_{\bs{r}} f(\bs{X})\end{bmatrix} + \frac{1}{2} A:\nabla\nabla f(\bs{X})
\]
and the difference of the generators of full gradient and stochastic gradient underdamped Langevin at $k$-th interation is denoted by 
\[
    \Delta \mathcal{L}_k f(\bs{X}) = (\mathcal{L}_k - \mathcal{L}) f(\bs{X}) = (\bs{F}_k(\bs{X}) - \bs{F}(\bs{X}))^T \begin{bmatrix} \nabla_{\bs{\theta}} f(\bs{X}) \\ \nabla_{\bs{r}} f(\bs{X})\end{bmatrix} = \langle \nabla V(\bs{\theta}) - n\nabla V_{I_k}(\bs{\theta}), \nabla_{\bs{r}} f(\bs{X}) \rangle
\]

For brevity, we write $\phi_k = \phi(\bs{X}^E_k)$, $\bs{F}^E_k = \bs{F}_k(\bs{X}^E_k)$, $\psi_k = \psi(\bs{X}^E_k)$ and $D^l\phi_k = (D^l\psi) (\bs{X}^E_k)$ where $(D^l\psi)(z)$ is the $l$-th order derivative. We write $(D^l\psi)[\bs{s}_1, \bs{s}_2,\cdots, \bs{s}_l]$ for derivative evaluated in the direction $\bs{s}_j, j=1,2,\cdots,l$. Define 
\[
    \bs{\delta}_k = \bs{X}^E_{k+1} - \bs{X}^E_k = h \bs{F}^E_k + \sqrt{h} G_k \bs{\eta}_{k+1}
\]

Under the assumptions of Theorem \ref{thm:mse}, we show that the vector field $F^E_k$ also has bounded momentum up to $p$-th order.

\begin{lemma} \label{lemma}
Under the assumption of Theorem \ref{thm:mse}, there exists a constant $M$ such that up to $\frac{p}{2}$-th order moments of random vector field $\bs{F}^E_k$ are bounded
\[
    \mathbb{E} \| \bs{F}^E_k \|^j_2 \le M, \, \forall j=0,1,2,\cdots,\frac{p}{2}, \, \forall k=0,1,2\cdots,
\]
\end{lemma}
\begin{proof}
It suffices to bound the highest moment, as all other lower order moments are bounded by the highest one by Holder's inequality.

First notice that 
\[
    \| \bs{F}^E_k\|_2 = \left\| \begin{bmatrix} \bs{r}^E_k \\ -\gamma \bs{r}^E_k - \nabla V_{I_k}(\bs{\theta}^E_k) \end{bmatrix} \right\|_2 \le  \sqrt{1 + \gamma^2} \|\bs{r}^E_k\|_2 + \|\nabla V_{I_k}(\theta^E_k)\|_2 
\]
Hence
\begin{align*}
    \mathbb{E} \| \bs{F}^E_k\|_2^{\frac{p}{2}} \le & \mathbb{E} \left(\sqrt{1 + \gamma^2} \|\bs{r}^E_k\|_2 + \|\nabla V_{I_k}(\theta^E_k)\|_2 \right)^\frac{p}{2} \\
    =& \mathbb{E} \left\{ \sum_{i=0}^{\frac{p}{2}} {\frac{p}{2} \choose i} \|\bs{r}^E_k\|_2^i \|\nabla V_{I_k}(\theta^E_k)\|_2 ^{\frac{p}{2}-i} \right\} \\
    =& \sum_{i=0}^{\frac{p}{2}} {\frac{p}{2} \choose i} \mathbb{E} \left[\|\nabla V_{I_k}(\theta^E_k)\|_2 ^{\frac{p}{2}-i}  \|\bs{r}^E_k\|_2^i \right] \\
    \le& \sum_{i=0}^{\frac{p}{2}} {\frac{p}{2} \choose i} \sqrt{\mathbb{E} \left[\|\nabla V_{I_k}(\theta^E_k)\|_2 ^{p-2i} \right]} \sqrt{\mathbb{E} \left[\|\bs{r}^E_k\|_2^{2i} \right]} \quad \mbox{(Cauchy-Schwarz inequality)}
\end{align*}
By assumption, we know each $\mathbb{E} \left[\|\nabla V_{I_k}(\theta^E_k)\|_2^l \right], \mathbb{E} \|\bs{r}^E_k\|^l_2, l=0,1,\cdots,p$ is bounded, so we conclude there exists a constant $M>0$ that bounds the $\frac{p}{2}$-th order moment of $\bs{F}^E_k, \forall k=0,1,\cdots,$
\end{proof}

Using Taylor's expansion for $\psi$, we have
\begin{equation*}
    \psi_{k+1} = \psi_k + D\psi_k [\bs{\delta}_k] + \frac{1}{2}D^2\psi_k[\bs{\delta}_k, \bs{\delta}_k] + \frac{1}{6}D^3\psi_k[\bs{\delta}_k, \bs{\delta}_k, \bs{\delta}_k] + R_{k+1}
\end{equation*}
where 
\[
    R_{k+1} = \left( \frac{1}{6}\int_0^1 s^3 D^4\psi(s\bs{X}^E_k + (1-s)\bs{X}^E_{k+1})ds \right)[\bs{\delta}_k, \bs{\delta}_k, \bs{\delta}_k, \bs{\delta}_k]
\]
is the remainder term. Therefore, we have
\begin{align}\label{eq:taylor}
    \psi_{k+1} =& \psi_k + h \mathcal{L}_k \psi_k + h^{\frac{1}{2}} D\psi_k [G_k\bs{\eta}_{k+1}] + h^{\frac{3}{2}} D^2\psi_k [\bs{F}^E_k, G_k\bs{\eta}_{k+1}] \\
    +& \frac{1}{2} h^2 D^2\psi_k [\bs{F}^E_k, \bs{F}^E_k] + \frac{1}{6}D^3\psi_k[\bs{\delta}_k, \bs{\delta}_k, \bs{\delta}_k] + r_{k+1} + R_{k+1} \nonumber
\end{align}
where
\[
    r_{k+1} = \frac{h}{2} \big(  D^2\psi_k [G_k\bs{\eta}_{k+1}, G_k\bs{\eta}_{k+1}] - A:\nabla\nabla \psi_k\big)
\]

Summing Equation \eqref{eq:taylor} ove the first $K$ terms, dividing by $Kh$ and use Poisson equation, we have
\begin{equation}
    \frac{1}{Kh} (\psi_K - \psi_0) = \frac{1}{K}\sum_{k=0}^{K-1}(\phi_k - \bar{\phi}) + \frac{1}{K}\sum_{k=0}^{K-1}\Delta\mathcal{L}_k \psi_k + \frac{1}{Kh} \sum_{i=1}^3 (M_{i,K} + S_{i, K}),
\end{equation}
where 
\begin{align*}
    M_{1, K} = \sum_{k=0}^{K-1} r_{k+1}, \, M_{2,K} = h^{\frac{1}{2}} \sum_{k=0}^{K-1} D\psi_k [G_k \bs{\eta}_{k+1}], \, M_{3,K} = h^{\frac{3}{2}} \sum_{k=0}^{K-1} D^2\psi_k [\bs{F}^E_k, G_k \bs{\eta}_{k+1}], \\
    S_{1,K} = \frac{h^2}{2} \sum_{k=0}^{K-1} D^2\psi_k [\bs{F}^E_k, \bs{F}^E_k], \, S_{2,K} = \sum_{k=0}^{K-1} R_{k+1}, \, S_{3,K} = \frac{1}{6}\sum_{k=0}^{K-1} D^3\psi_k[\bs{\delta}_k, \bs{\delta}_k, \bs{\delta}_k]
\end{align*}
Furthermore, it will be convenient to decompose
\[
    S_{3,K} = M_{0,K} + S_{0,K}
\]
where 
\begin{align*}
    S_{0,K} =& h^2 \sum_{k=0}^{K-1} \big( hD^3\psi_k [\bs{F}^E_k, \bs{F}^E_k, \bs{F}^E_k] + 3 D^3\psi_k[\bs{F}^E_k, G_k\bs{\eta}_{k+1}, G_k\bs{\eta}_{k+1}] \big) \\
    M_{0,K} =& h^{\frac{3}{2}} \sum_{k=0}^{K-1} \big( D^3\psi_k [G_k\bs{\eta}_{k+1}, G_k\bs{\eta}_{k+1}, G_k\bs{\eta}_{k+1}] + 3h D^3\psi_k[\bs{F}^E_k, \bs{F}^E_k, G_k\bs{\eta}_{k+1}] \big)
\end{align*}

Rearrange terms in Equation \eqref{eq:taylor}, square on both sides, use Cauchy-Schwarz inequality and take expectation, we have
\begin{align}
    \mathbb{E}\big( \hat{\phi}_K - \bar{\phi} \big)^2 \le& C \left[  \mathbb{E} \frac{(\psi_K - \psi_0)^2}{(Kh)^2} + \frac{1}{K^2} \mathbb{E} \left( \sum_{k=0}^{K-1}(\Delta \mathcal{L}_k \psi_k) \right)^2  + \frac{1}{(Kh)^2} \sum_{i=0}^2 \mathbb{E}S^2_{i,K} + \frac{1}{(Kh)^2} \sum_{i=0}^3 \mathbb{E}M^2_{i,K}  \right] \nonumber \\
    =& C \left[  \mathbb{E} \frac{(\psi_K - \psi_0)^2}{T^2} + \frac{1}{K^2} \mathbb{E} \left( \sum_{k=0}^{K-1}(\Delta \mathcal{L}_k \psi_k) \right)^2  + \frac{1}{T^2} \sum_{i=0}^2 \mathbb{E} S^2_{i,K} + \frac{1}{T^2} \sum_{i=0}^3 \mathbb{E} M^2_{i,K}  \right]
\end{align}
where $T = kh$, the corresponding time of the underlying continuous dynamics. 

We now show how each term is bounded. By boundedness of $\psi$, we have
\[
   \mathbb{E} \frac{(\psi_K - \psi_0)^2}{T^2} \le \frac{4\|\psi\|_\infty^2}{T^2} = \mathcal{O}(\frac{1}{T^2})
\] 

The second term $\frac{1}{K^2} \mathbb{E} \big( \sum_{k=0}^{K-1}(\Delta \mathcal{L}_k \psi_k) \big)^2$ is critical in showing the advantage of EWSG, and we will show how to derive its bound in detail later.

The technique we use to bound $\frac{1}{T^2} \mathbb{E} S_{i,K}^2, i=0,1,2$ are all similar, we will first show an upper bound for $|S_{i,K}|$ in terms of powers of $\|\bs{F}^E_k\|$, then take square and expectation, and finally expand squares and use Lemma \ref{lemma} extensively to derive bounds. As a concrete example, we will show how to bound $\frac{1}{T^2}\mathbb{E} S_{0, K}^2$. Other bounds follow in a similar fashion and details are omitted.

To bound the term containing $S_{0,K}$, we first note that
\begin{align*}
    |S_{0,K}| \le& h^2 \sum_{k=0}^{K-1} \big( h|D^3\psi_k [\bs{F}^E_k, \bs{F}^E_k, \bs{F}^E_k]| + 3 |D^3\psi_k[\bs{F}^E_k, G_k\bs{\eta}_{k+1}, G_k\bs{\eta}_{k+1}]| \big) \\
    \le & h^2 \|D^3\psi\|_\infty \sum_{k=0}^{K-1} \big( h \|\bs{F}^E_k\|_2^3 + 3 \|\bs{F}^E_k\|_2 \|G_k\bs{\eta}_{k+1}\|^2_2 \big)
\end{align*}
Square both sides of the above inequality and take expectation, we obtain
\begin{align}\label{eq:S_0}
    & \frac{1}{T^2}\mathbb{E} |S_{0,K}|^2 \\
    \le&  \frac{h^4}{T^2} \|D^3\psi\|_\infty^2 \mathbb{E}\big( \sum_{k=0}^{K-1} h \|\bs{F}^E_k\|_2^3 + 3 \|\bs{F}^E_k\|_2 \|G_k\bs{\eta}_{k+1}\|^2_2 \big)^2 \nonumber\\
    \le& \frac{h^4}{T^2} \|D^3\psi\|_\infty^2 K  \sum_{k=0}^{K-1} \mathbb{E}( h \|\bs{F}^E_k\|_2^3 + 3 \|\bs{F}^E_k\|_2 \|G_k\bs{\eta}_{k+1}\|^2_2)^2 \quad \mbox{(Cauchy-Schwarz inequality)} \nonumber\\
    =& \frac{h^4}{T^2} \|D^3\psi\|_\infty^2 K  \sum_{k=0}^{K-1} \mathbb{E}[ h^2 \|\bs{F}^E_k\|_2^6 + 6 \|\bs{F}^E_k\|_2^4 \|G_k\bs{\eta}_{k+1}\|^2_2 + 9 \|\bs{F}^E_k\|_2^2 \|G_k\bs{\eta}_{k+1}\|^2_4] \nonumber\\
    =& \frac{h^4}{T^2} \|D^3\psi\|_\infty^2 K  \sum_{k=0}^{K-1} h^2 \mathbb{E} \|\bs{F}^E_k\|_2^6 + 6\mathbb{E}\|\bs{F}^E_k\|_2^4 \mathbb{E}\|G_k\bs{\eta}_{k+1}\|^2_2 + 9 \mathbb{E}\|\bs{F}^E_k\|_2^2 \mathbb{E}\|G_k\bs{\eta}_{k+1}\|^2_4 \nonumber\\
    =& \frac{1}{T^2} \mathcal{O}(K^2h^4) \nonumber\\
    =&\mathcal{O}(h^2) \nonumber
\end{align}

To bound the term containing $S_{1,K}$ and $S_{2, K}$, we have
\begin{align*}
    |S_{1,K}| \le& \frac{h^2}{2} \sum_{k=0}^{K-1} \|D^2\psi\|_\infty \|\bs{F}^E_k\|_2^2  \\
    |S_{2,K}| \le& \frac{1}{24} \|D^4\psi\|_\infty \sum_{k=0}^{K-1}\|\bs{\delta}_k\|_2^4 
    \le \frac{1}{24} h^2 \|D^4\psi\|_\infty \sum_{k=0}^{K-1} \|\sqrt{h}\bs{F}^E_k + G_k \bs{\eta}_{k+1}\|_2^4
\end{align*}
Then we can obtain the following bound in a similar fashion as in Equation \eqref{eq:S_0}
\begin{align*}
    \frac{1}{T^2} \mathbb{E} S_{1,K}^2 =& \mathcal{O}(h^2) \\
    \frac{1}{T^2} \mathbb{E} S_{2,K}^2 =& \mathcal{O}(h^2)
\end{align*}

Now we will use martingale argument to bound $\frac{1}{T^2}\mathbb{E}M^2_{i,K}, i=0,1,2,3$.  There are two injected randomness at $k$-th iteration, the Gaussian noise $\bs{\eta}_{k+1}$ and the stochastic gradient term determined by the stochastic index $I_k$. Denote the sigma algebra at $k$-th iteration by $\mathcal{F}_k$. For both SGHMC and EWSG we have
\[
    \bs{\eta}_{k+1} \perp \mathcal{F}_k \mbox{ and } I_k \perp \bs{\eta}_{k+1}
\]
hence 
\begin{align*}
\mathbb{E} [\bs{\eta}_{k+1} | \mathcal{F}_k] =& \bs{0} \\ \mathbb{E}[D^3\psi_k[G_k\bs{\eta}_{k+1}, G_k\bs{\eta}_{k+1}, G_k\bs{\eta}_{k+1}]| \mathcal{F}_k] =& 0\\
\mathbb{E} [D^2\psi_k [\bs{F}^E_k, G_k\bs{\eta}_{k+1}]|\mathcal{F}_k] =& 0\\ \mathbb{E} [D^3\psi_k [\bs{F}^E_k, \bs{F}^E_k, G_k\bs{\eta}_{k+1}]|\mathcal{F}_k] =& 0
\end{align*}

Therefore, it is clear that $M_{i,K}, i=0,1,2,3$ are all martingales. Due to martingale properties, we have
\begin{align*}
    \frac{1}{T^2} \mathbb{E}M_{0,K}^2 =& \frac{h^3}{T^2} \sum_{k=0}^{K-1} \mathbb{E} \big( D^3\psi_k [G_k\bs{\eta}_{k+1}, G_k\bs{\eta}_{k+1}, G_k\bs{\eta}_{k+1}] + 3h D^3\psi_k[\bs{F}^E_k, \bs{F}^E_k, G_k\bs{\eta}_{k+1}] \big)^2\\
    =& \frac{1}{T^2}\mathcal{O}(h^3K) = \mathcal{O}(\frac{h^2}{T}) \\
    \frac{1}{T^2}\mathbb{E} M_{1, K}^2 =& \frac{1}{T^2}\sum_{k=0}^{K-1} \mathbb{E} r_{k+1}^2 = \frac{1}{T^2}\mathcal{O}(h^2 K) = \mathcal{O}(\frac{h}{T})\\
    \frac{1}{T^2}\mathbb{E} M_{2,K}^2 =& \frac{h}{T^2}\sum_{k=0}^{K-1} \mathbb{E} (D\psi_k [G_k \bs{\eta}_{k+1}])^2 = \frac{1}{T^2}\mathcal{O}(hK) = \mathcal{O}(\frac{1}{T})\\
    \frac{1}{T^2}\mathbb{E} M^2_{3,K} =& \frac{1}{T^2}h^3 \sum_{k=0}^{K-1} \mathbb{E}(D^2\psi_k [\bs{F}^E_k, G_k \bs{\eta}_{k+1}])^2 = \frac{1}{T^2}\mathcal{O}(h^3 K) = \mathcal{O}(\frac{h^2}{T})
\end{align*}

We now collect all bounds derived so far and obtain
\begin{align} \label{eq:final}
    \mathbb{E}\big( \hat{\phi}_K - \bar{\phi} \big)^2 \le& C \left[\mathcal{O}(\frac{1}{T^2}) + \frac{1}{K^2} \mathbb{E} \left( \sum_{k=0}^{K-1}(\Delta \mathcal{L}_k \psi_k) \right)^2 + \mathcal{O}(h^2) + \mathcal{O}(\frac{h}{T}) + \mathcal{O}(\frac{1}{T}) + \mathcal{O}(\frac{h^2}{T})\right] \nonumber\\
    \le& C \left[ \mathcal{O}(\frac{1}{T}) + \frac{1}{K^2} \mathbb{E} \left( \sum_{k=0}^{K-1}(\Delta \mathcal{L}_k \psi_k) \right)^2 + \mathcal{O}(h^2)\right] 
\end{align}
In the above inequality, we use $\frac{1}{T^2} < \frac{1}{T}$ and $\frac{h}{T} \le \frac{1}{T}, \frac{h^2}{T} \le \frac{1}{T}$ as typically we assume $T \gg 1$ and $h \ll 1$ in non-asymptotic analysis.

Now we focus on the remaining term $ \frac{1}{K^2} \mathbb{E} \big( \sum_{k=0}^{K-1}\Delta \mathcal{L}_k \psi_k \big)^2$. 
For SGHMC, we have that $\mathbb{E}[\Delta\mathcal{L}_k \psi_k|\mathcal{F}_k]= 0$, hence $\sum_{k=0}^{K-1}\Delta \mathcal{L}_k \psi_k $ is a martingale. By martingale property, we have
\[
    \frac{1}{K^2} \mathbb{E} \left( \sum_{k=0}^{K-1}\Delta \mathcal{L}_k \psi_k \right)^2 = \frac{1}{K^2} \sum_{k=0}^{K-1} \mathbb{E}(\Delta \mathcal{L}_k \psi_k)^2
\]
For EWSG, $\sum_{k=0}^{K-1}\Delta \mathcal{L}_k \psi_k $ is no longer a martingale, but we still have the following
\begin{align}
    \frac{1}{K^2} \mathbb{E} \left( \sum_{k=0}^{K-1}\Delta \mathcal{L}_k \psi_k \right)^2 =& \frac{1}{K^2} \sum_{k=0}^{K-1} \mathbb{E}(\Delta\mathcal{L}_k \psi_k)^2 + \frac{2}{K^2} \sum_{i < j} \mathbb{E} (\Delta\mathcal{L}_i \psi_i) (\Delta\mathcal{L}_j \psi_j) \nonumber\\
    =& \frac{1}{K^2} \sum_{k=0}^{K-1} \mathbb{E}(\Delta\mathcal{L}_k \psi_k)^2 + \frac{2}{K^2} \sum_{i < j} \mathbb{E}[  (\Delta\mathcal{L}_i \psi_i) \mathbb{E} [\Delta\mathcal{L}_j \psi_j | \mathcal{F}_j]]  \label{eq:cross_term}
\end{align}

For the term $\mathbb{E} [\Delta\mathcal{L}_j \psi_j | \mathcal{F}_j]$, we have
\[
    \mathbb{E} [\Delta\mathcal{L}_j \psi_j | \mathcal{F}_j]
    = \mathbb{E}[\langle \nabla V(\bs{\theta}^E_j) - n\nabla V_{I_j}(\bs{\theta}^E_j), \nabla_{\bs{r}}\psi_j\rangle | \mathcal{F}_j]
    = \langle \mathbb{E}[\nabla V(\bs{\theta}^E_j) - n\nabla V_{I_j}(\bs{\theta}^E_j) | \mathcal{F}_j], \nabla_{\bs{r}}\psi_j\rangle
\]
as $\psi_j \in \mathcal{F}_j$. Then by Cauchy-Schwarz inequality, boundedness of $\psi$ and the fact $\| \nabla V(\bs{\theta}^E_j) - \mathbb{E}[n\nabla V_{I_j}(\bs{\theta}^E_j) | \mathcal{F}_j] \|_2 = \mathcal{O}(h)$ as shown in the proof of Theorem \ref{thm:local}, we conclude $\mathbb{E} [\Delta\mathcal{L}_j \psi_j | \mathcal{F}_j] = \mathcal{O}(h)$.

Now plug the above result in Equation \eqref{eq:cross_term}, we have
\begin{align*}
    \frac{1}{K^2} \mathbb{E} \left( \sum_{k=0}^{K-1}\Delta \mathcal{L}_k \psi_k \right)^2 =& \frac{1}{K^2} \sum_{k=0}^{K-1} \mathbb{E}(\Delta\mathcal{L}_k \psi_k)^2 + \frac{2}{K^2} \sum_{i < j} \mathbb{E}[  (\Delta\mathcal{L}_i \psi_i) \mathbb{E} [\Delta\mathcal{L}_j \psi_j | \mathcal{F}_j]] \\
    =& \frac{1}{K^2} \sum_{k=0}^{K-1} \mathbb{E}(\Delta\mathcal{L}_k \psi_k)^2 + \frac{2}{K^2} \sum_{i < j} \mathbb{E}[\Delta\mathcal{L}_i \psi_i]  \mathcal{O}(h) \\
    =& \frac{1}{K^2} \sum_{k=0}^{K-1} \mathbb{E}(\Delta\mathcal{L}_k \psi_k)^2 + \frac{2}{K^2} \sum_{i < j}  \mathcal{O}(h^2) \\
    =& \frac{1}{K^2} \sum_{k=0}^{K-1} \mathbb{E}(\Delta\mathcal{L}_k \psi_k)^2 + \frac{2}{K^2} \sum_{i < j}  \mathcal{O}(h^2) \\
    =& \frac{1}{K^2} \sum_{k=0}^{K-1} \mathbb{E}(\Delta\mathcal{L}_k \psi_k)^2 + \mathcal{O}(h^2)
\end{align*}

Combine both cases of SGHMC and EWSG, we obtain
\[
    \frac{1}{K^2} \mathbb{E} \left( \sum_{k=0}^{K-1}\Delta \mathcal{L}_k \psi_k \right)^2 = \frac{1}{K^2} \sum_{k=0}^{K-1} \mathbb{E}(\Delta\mathcal{L}_k \psi_k)^2 + \mathcal{O}(h^2)
\]
Note that $\mathcal{O}(h^2)$ term will later be combined with other error terms with the same order.

The final piece is to bound $\frac{1}{K^2} \sum_{k=0}^{K-1} \mathbb{E}(\Delta\mathcal{L}_k \psi_k)^2$, and we have
\begin{align*}
    \frac{1}{K^2} \sum_{k=0}^{K-1} \mathbb{E}(\Delta\mathcal{L}_k \psi_k)^2
    &= \frac{1}{K^2}\sum_{k=0}^{K-1} \mathbb{E}\langle \nabla V(\bs{\theta}^E_k) - n\nabla V_{I_k}(\bs{\theta}^E_k), \nabla_{\bs{r}}\psi_k\rangle^2\\ 
    &\le \frac{1}{K^2}\sum_{k=0}^{K-1} \mathbb{E}[\|\nabla V(\bs{\theta}^E_k) - n\nabla V_{I_k}(\bs{\theta}^E_k) \|^2_2 \cdot  \|\nabla_{\bs{r}}\psi_k\|^2_2] \quad \mbox{(Cauchy-Schwarz inequality)} \\
    &\le \frac{M_3^2}{K^2}\sum_{k=0}^{K-1} \mathbb{E}[\|\nabla V(\bs{\theta}^E_k) - n\nabla V_{I_k}(\bs{\theta}^E_k) \|^2_2] \\
    &= \frac{M_3^2}{K^2}\sum_{k=0}^{K-1} \mathbb{E}[\mathbb{E}[\|\nabla V(\bs{\theta}^E_k) - n\nabla V_{I_k}(\bs{\theta}^E_k) \|^2_2 \, | \, \mathcal{F}_k]] \\
    &\le \frac{2M_3^2}{K^2}\sum_{k=0}^{K-1} \mathbb{E}[\, \underbrace{\mathbb{E}[ \, \|\nabla V(\bs{\theta}^E_k) - \mathbb{E}[n\nabla V_{I_k}(\bs{\theta}^E_k) \,|\, \mathcal{F}_k] \|^2 \, | \, \mathcal{F}_k]}_{Q_1} ] \\
    &+ \underbrace{\mathbb{E} [ \, \| \mathbb{E}[n\nabla V_{I_k}(\bs{\theta}^E_k) \,|\, \mathcal{F}_k] - n\nabla V_{I_k}(\bs{\theta}^E_k)\|^2_2 \, | \, \mathcal{F}_k]}_{Q_2}]
\end{align*}

The term $Q_1$ captures the bias of stochastic gradient. For SGHMC, uniform gradient subsamping leads to an unbiased gradient estimator, so $Q_1 = 0$ for SGHMC. For EWSG, same as in the proof of Theorem 2, we have that 
\[
    \mathbb{E}\left[ \, \|\nabla V(\bs{\theta}^E_k) - \mathbb{E}[n\nabla V_{I_k}(\bs{\theta}^E_k) \,|\, \mathcal{F}_k] \|^2 \, | \, \mathcal{F}_k\right] = \mathcal{O}(h^2)
\]
Combining two cases, we have 
\[
    Q_1 = \mathcal{O}(h^2)
\]

For a random vector $\bs{v}$ with mean $\mathbb{E}[\bs{v}] = \bs{0}$, we have 
\[
    \mathbb{E} [\|\bs{v}\|^2] = \mathbb{E}\left[\Tr[\bs{v}\bs{v}^T]\right] = \Tr \left[\mathbb{E} [\bs{v}\bs{v}^T] \right] = \Tr\left[\mbox{cov}(\bs{v})\right]
\]
where $\mbox{cov}(\bs{v})$ is the covariance matrix of random vector $\bs{v}$. Therefore, we have that 
\[
    Q_2 = \Tr\left[\mbox{cov}(n\nabla V_{I_k} | \mathcal{F}_k)\right],
\]
i.e., $Q_2$ is the trace of the covariance matrix of stochastic gradient estimate conditioned on current filtration $\mathcal{F}_k$.

Combining $Q_1$ and $Q_2$, we have that 
\begin{align*}
    \frac{1}{K^2} \mathbb{E} \left( \sum_{k=0}^{K-1}\Delta \mathcal{L}_k \psi_k \right)^2 \le& \frac{2M_3^2}{K^2} \sum_{k=0}^{K-1} \left[\mathbb{E}[\Tr[\mbox{cov}(n\nabla V_{I_k} | \mathcal{F}_k)]] + \mathcal{O}(h^2)\right] \\
    =& \frac{2M_3^2 h}{T} \frac{\sum_{k=0}^{K-1} \mathbb{E}[\Tr[\mbox{cov}(n\nabla V_{I_k} | \mathcal{F}_k)]]}{K} + \mathcal{O}(\frac{h^3}{T})
\end{align*}

Now plug this bound into Equation \eqref{eq:final} and we obtain
\[
    \mathbb{E}\big( \hat{\phi}_K - \bar{\phi} \big)^2
    \le C_1 \frac{1}{T} + C_2 \frac{h}{T} \frac{\sum_{k=0}^{K-1} \mathbb{E}\left[\Tr[\mbox{cov}(n\nabla V_{I_k} | \mathcal{F}_k)]\right]}{K} + C_3 h^2
\]
for some constants $C_1, C_2, C_3 > 0$ depending on $M_1, M_2, M_3$.
\end{proof}

\section{Mini Batch Version of EWSG}\label{ewsg:minibatch}
When mini batch size $b > 1$, for each mini batch $\{i_1, i_2, \cdots, i_b\}$, we use $\frac{n}{b}\sum_{j=1}^b \nabla V_{i_j}$ to approximate full gradient $\nabla V$, and assign the mini batch $\{i_1, i_2, \cdots, i_b\}$ probability $p_{i_1i_2,\cdots,i_b}$. We can easily extend the transition probability of $b=1$ to general $b$, simply by replacing $n\nabla V_i$ with $\frac{n}{b}\sum_{j=1}^b \nabla V_{i_j}$ and end up with
\[
    \tilde{P}(\bs{\theta}_{k+1}, \bs{r}_{k+1} | \bs{\theta}_k, \bs{r}_k) = \delta(\bs{\theta}_{k+1} = \bs{\theta}_k + \bs{r}_k h)\times\]
    \[ \sum_{i_1, i_2, \cdots, i_b} p_{i_1 i_2 \cdots i_b} \Phi \left(\bs{x} + n\bs{a}_{i_1 i_2 \cdots i_b}\right) \frac{1}{\sigma \sqrt{h}}
\]
where 
\[ \bs{x} = \frac{\bs{r}_{k+1} - \bs{r}_k + h\gamma \bs{r}_k}{\sigma \sqrt{h}},
\quad \bs{a}_{i_1 i_2 \cdots i_b}= \frac{\sqrt{h}}{\sigma} \frac{1}{b}\sum_{j=1}^b \nabla V_{i_j}(\bs{\theta}_k)
\]
Therefore, to match the transition probability of underdamped Langevin dynamics with stochastic gradient and full gradient, we let $p_{i_1 i_2 \cdots i_b} =$
{\small
\[
\displaystyle \frac{1}{Z} \exp\left\{\frac{1}{2} \left[\|\bs{x} + n \bs{a}_{i_1 i_2 \cdots i_b}\|^2 - \|\bs{x} + \displaystyle\sum_{i_1 i_2 \cdots i_b} \bs{a}_{i_1 i_2 \cdots i_b}\|^2 \right] \right\}\]
}
where $Z$ is a normalization constant.

To sample multidimensional random data indices $I_1, \cdots, I_b$ from $p_{i_1 i_2 \cdots i_b}$, we again use a Metropolis chain, whose acceptance probability only depends on $a_{i_1 i_2 \cdots i_b}$ and $a_{j_1 j_2 \cdots j_b}$ but not the full gradient.

\section{EWSG Version for Overdamped Langevin} \label{ewsg:sgld}
Overdamped Langevin equation is the following SDE
\[
    d\bs{\theta}_t = -\nabla V(\bs{\theta}_t) dt + \sqrt{2} d\bs{B}_t
\]
where $V(\bs{\theta}) = \sum_{i=1}^n V_i(\bs{\theta})$ and $B_t$ is a $d$-dimensional Brownian motion. The Euler-Maruyama discretization is 
\[
    \bs{\theta}_{k+1} = \bs{\theta}_k - h \nabla V(\bs{\theta}_k) + \sqrt{2h} \bs{\xi}_{k+1}
\]
where $\bs{\xi}_{k+1}$ is a $d$-dimensional random Gaussian vector. When stochastic gradient is used, the above numerical schedme turns to
\[
        \bs{\theta}_{k+1} = \bs{\theta}_k - h \nabla V_{I_k}(\bs{\theta}_k) + \sqrt{2h} \bs{\xi}_{k+1}
\]
where $I_k$ is the datum index used in $k$-th iteration to estimate the full gradient.

Denote $\bs{x} = \frac{\bs{\theta}_{k+1} - \bs{\theta}_k}{\sqrt{2h}}$ and $\bs{a}_i = \frac{\sqrt{h}\nabla V_i(\bs{\theta}_k)}{\sqrt{2}}$. If we set
\[
    p_i = \mathbb{P}(I_k = i) \propto\exp \big\{ -\frac{\|\bs{x} + \sum_{j=1}^n \bs{a}_j\|^2}{2} + \frac{\|\bs{x} + n\bs{a}_i\|^2}{2} \big\}
\]
and follow the same steps in Sec.\ref{sec:method:EWSG}, we will see the transition kernel of the full gradient method being approximated by that of the stochastic gradient version.

\section{Variance Reduction (VR)}\label{variance reduction}
We have seen that when step size $h$ is large, EWSG still introduces extra variance. To further  mitigate this inaccuracy, we provide in this section a complementary variance reduction technique.

Locally (i.e., conditioned on the state of the system at the current step),  we have increased variance
\begin{align}
  \mbox{cov}[\bs{r}_{k+1}|\bs{r}_k] &= \mathbb{E}[\mbox{cov}[\bs{r}_{k+1}|I]] + \mbox{cov}[\mathbb{E}[\bs{r}_{k+1}|I]]  \nonumber\\
  &= h(\Sigma_{k+1}^2 + h \, \mbox{cov}[n\nabla V_I(\bs{\theta}_k)])
  \label{eq:varianceAccumulation}
\end{align}
where $\Sigma^2_{k+1} = \frac{1}{h} \mathbb{E}[\mbox{cov}[\bs{r}_{k+1}|I]]$.
The extra randomness due to the randomness of the index $I$ enters the parameter space through the coupling of $\bs{\theta}$ and $\bs{r}$ and eventually deviates the stationary distribution from that of the original dynamics. Adopting the perspective of modified equation \citep{borkar1999strong,mandt2017stochastic,li2017stochastic}, we model this as an enlarged diffusion coefficient. To correct for this enlargement and still sample from the correct distribution, we can either, in each step, shrink the size of intrinsic noise to $\Sigma_{k} \in \mathbb{R}^{d \times d}$ such that
$
\sigma^2 I = \Sigma_{k}^2 + h\mbox{cov}[n \nabla V_I(\bs{\theta}_{k-1})]
$,
or alternatively increase the dissipation. More precisely, due to the matrix version fluctuation dissipation theorem $\Sigma^2 = 2\Gamma T$, one could instead increase the friction coefficient $\Gamma \in \mathbb{R}^{d \times d}$ rather than shrinking the intrinsic noise. The second approach is computationally more efficient because it no longer requires square-rooting / Cholesky decomposition of (possibly large-scale) matrices. Therefore, in each step, we set 
\[
  \Gamma_k = \frac{1}{2T}(\sigma^2 I + h\mbox{cov}[n\nabla V_I(\bs{\theta}_{k-1})]).
\]

Accurately computing $\mbox{cov}[n\nabla V_I(\bs{\theta}_{k-1})]$ is expensive as it requires running $I$ through ${1,\cdots,n}$, which defeats the purpose of introducing a stochastic gradient. To downscale the computation cost from $\mathcal{O}(n)$ to $\mathcal{O}(1)$, we use an SVRG type estimation of the this variance instead. More specifically, we periodically compute $\mbox{cov}[n\nabla V_I(\bs{\theta}_{k-1})]$ only every $L$ data passes, in an outer loop. In every  iteration of an inner loop, which integrates the Langevin, an estimate of $\mbox{cov}[n\nabla V_I(\bs{\theta}_{k-1})]$ is updated in an SVRG fashion.
sof
See Algorithm~\ref{alg:EWSG-VR} for detailed description. We refer variance reduced variant of EWSG as EWSG-VR.

\begin{algorithm}[h]
  \caption{EWSG-VR}\label{alg:EWSG-VR}
\begin{algorithmic}[1]
  \STATE {\bfseries Input:} \{number of data terms $n$, gradient functions $\nabla V_i(\cdot)$, step size $h$, number of data passes K, period of variance calibration $L$, index chain length $M$, friction and noise coefficients $\gamma$ and $\sigma$\}
  \STATE initialize $\bs{\theta}_0, \bs{r}_0, \gamma_0 = \gamma$
  \STATE initialize inner loop index $k=0$
  \FOR{$l = 1, 2, \cdots, K$}
  \IF{$(l-1) \mod L = 0$}
  \STATE compute $\bs{m}_1 \gets \mathbb{E}_I[n\nabla V_I(\bs{\theta}_k)]$,\quad $\bs{m}_2 \gets \mathbb{E}_I[n^2\nabla V_I(\bs{\theta}_k) \nabla V_I(\bs{\theta}_k)^T]$
  \STATE $\bs{\omega} \gets \bs{\theta}_k$
  \ELSE 
  \FOR{$ t = 1,2, \cdots, \lceil \frac{n}{M + 1}\rceil$}
  \STATE $i \gets$ uniformly sampled from ${1,\cdots,n}$, \, compute and store $n\nabla V_i(\bs{\theta}_k)$
    \FOR{$m = 1, 2, \cdots, M$}
  \STATE $j \gets$ uniformly sampled from ${1,\cdots,n}$, \, compute and store $n\nabla V_j(\bs{\theta}_k)$
  \STATE $i \gets j$ with probability in Equation \ref{eq:transition}
  \ENDFOR
    \STATE update $(\bs{\theta}_{k+1}, \bs{r}_{k+1}) \gets (\bs{\theta}_k, \bs{r}_k)$ according to Equation \ref{EM}, using $n\nabla V_i(\bs{\theta}_k)$ as gradient and $\Gamma_k$ as friction 
  \STATE $\bs{m}_1 \gets \bs{m}_1 + \nabla V_i(\bs{\theta}_k) - \nabla V_i(\bs{\omega})$
  \STATE $\bs{m}_2 \gets \bs{m}_2 + n\nabla V_i(\bs{\theta}_k)\nabla V_i(\bs{\theta}_k)^T - n\nabla V_i(\bs{\omega})\nabla V_i(\bs{\omega})^T$
  \STATE $\mbox{covar} \gets \bs{m}_2 - \bs{m}_1 \bs{m}_1^T$
  \STATE $\Gamma_{k+1} \gets \frac{1}{2T}(\sigma^2 \text{I} + h \, \mbox{covar})$
  \STATE $k \gets k + 1$
  \ENDFOR
  \ENDIF
  \ENDFOR
\end{algorithmic}
\end{algorithm}

\begin{figure}[h]
    \centering
    \includegraphics[width=0.4\columnwidth]{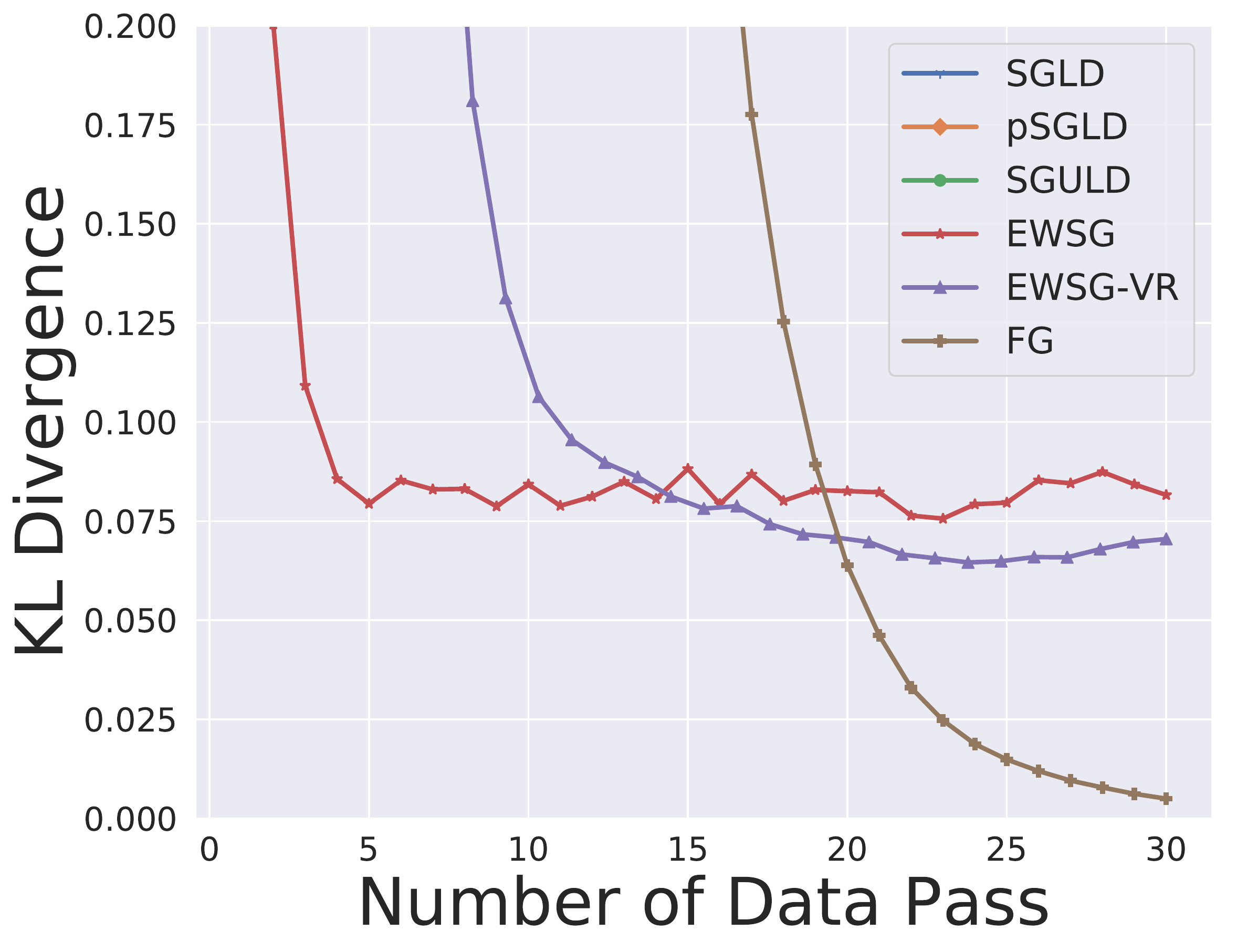}
    \caption{KL divergence}\label{fig:supp_gaussian}
\end{figure}

To demonstrate the performance of EWSG-VR, we reuse the setup of simple Gaussian example in subsection \ref{subsec:simple_gaussian}. As shown in Algorithm \ref{alg:EWSG-VR}, the only hyper-parameter of EWSG-VR additional to EWSG is the period of variance calibration, for which we set $L =1$. All other hyper-parameters (e.g. step size $h$, friction coefficient  $\gamma$) are set the same as EWSG. We also run underdamped Langevin dynamics with full gradient (FG) using the same hyper-parameters of EWSG. We plot the KL divergence in Figure \ref{fig:supp_gaussian}. We see that EWSG-VR further reduces variance and achieves better statistical accuracy measured in KL divergence. Although EWSG-VR periodically use full data set to calibrate variance estimation, it is still significantly faster than the full gradient version. Note that KL divergence of SGLD, pSGLD and SGHMC are too large so that we can not even see them in Figure \ref{fig:supp_gaussian}

We also consider applying EWSG-VR to Bayesian logistic regression problems. We run experiments on two standard classification data sets \texttt{parkinsons} \footnote{https://archive.ics.uci.edu/ml/datasets/parkinsons}, \texttt{pima}\footnote{https://archive.ics.uci.edu/ml/datasets/diabetes} from UCI repository \citep{lichman2013uci}. 

\begin{figure}[h]
    \begin{subfigure}[t]{0.49\textwidth}
		\centering
		\includegraphics[width=\textwidth]{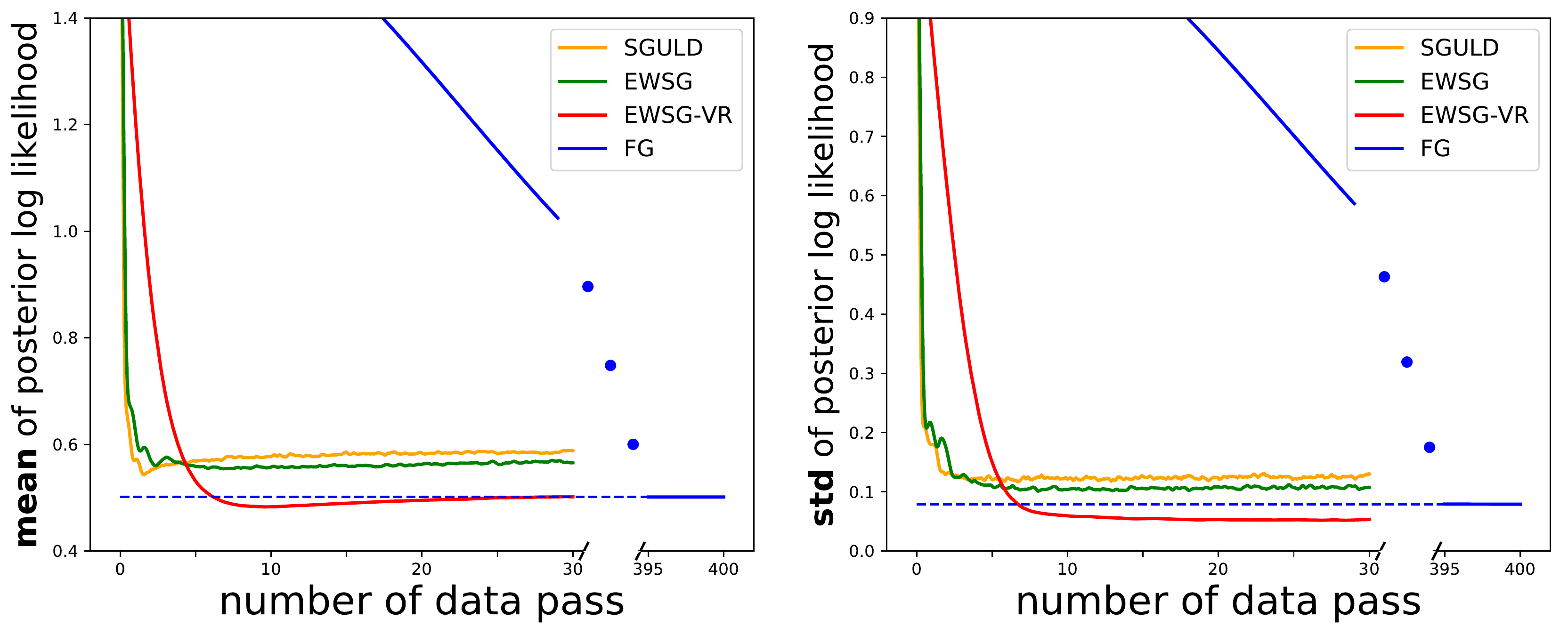}
		\caption{parkinsons} \label{parkinsons}	
	\end{subfigure}
    \begin{subfigure}[t]{0.49\textwidth}
		\centering
		\includegraphics[width=\textwidth]{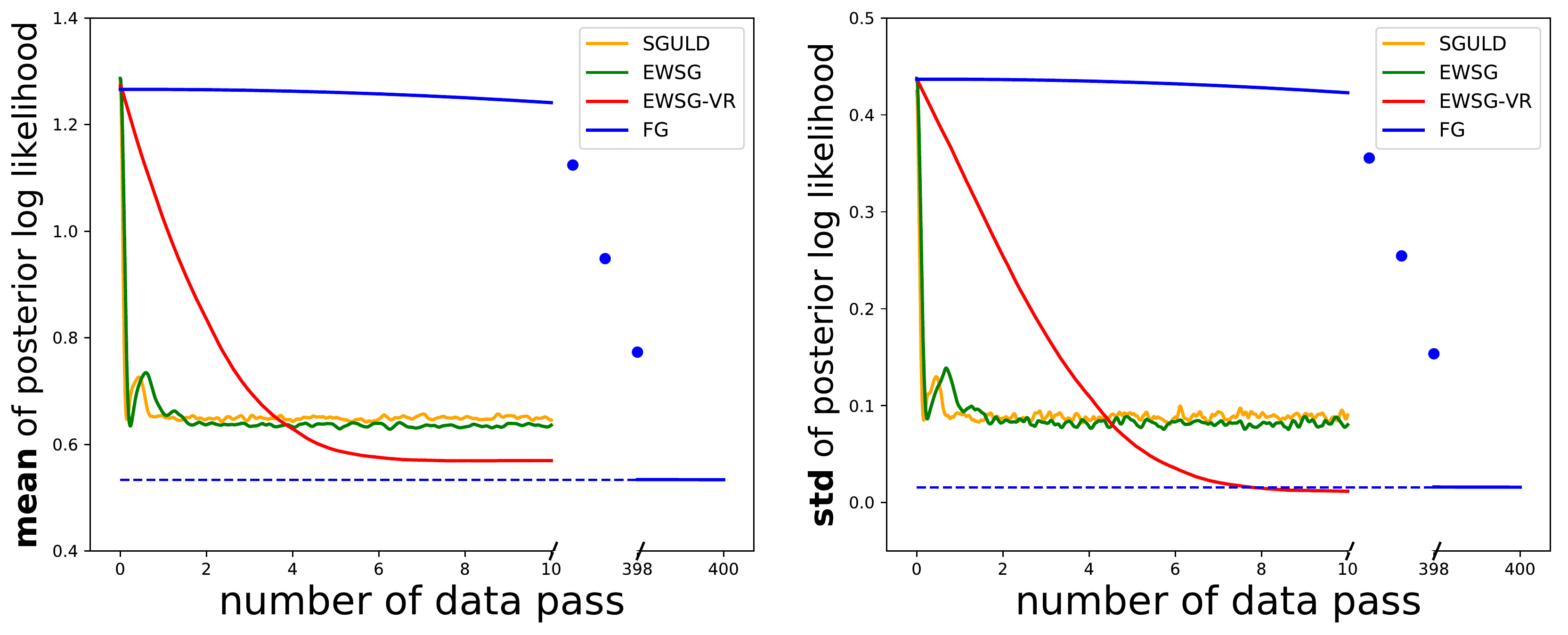}
		\caption{pima} \label{pima}	
	\end{subfigure}

    \caption{Posterior prediction of mean (\textit{left}) and standard deviation (\textit{right}) of log likelihood on test data set generated by SGHMC, EWSG and EWSG-VR on two Bayesian logistic regression tasks. Statistics are computed based on 1000 independent simulations. Minibatch size $b=1$ for all methods except FG. $M=1$ for EWSG and EWSG-VR.}\label{fig:classficiation}
\end{figure}

From Figure \ref{fig:classficiation}, we see stochastic gradient methods (SGHMC, EWSG and EWSG-VR) only take tens of data passes to converge while full gradient version (FG) requires hundreds of data passes to converge. Compared with SGHMC, EWSG produces closer results to FG for which we treat as ground truth, in terms of statistical accuracy. With variance reduction, EWSG-VR is able to achieve even better performance, significantly improving the accuracy of the prediction of mean and standard deviation of log likelihood. It, however, converges slower than EWSG without VR.

One downside of EWSG-VR is that it periodically use whole data set to calibrate variance estimation, so it may not be suitable for very large data sets (e.g. Covertype data set used in subsection \ref{subsec:blr}) for which stochastic gradient methods could converge within one data pass.

\section{Additional Experiments} \label{sec:additional_exp}
\subsection{A Misspecified Gaussian Case} \label{subsec:misspecified}
In this subsection, we follow the same setup as in \citep{bardenet2017markov} and study a misspecified Gaussian model where one fits a one-dimensional normal distribution $p(\theta) = \mathcal{N}(\theta | \mu_0, \sigma_0^2)$ to $10^5$ i.i.d points drawn according to $X_i \sim \log \mathcal{N}(0, 1)$, and flat prior is assigned $p(\mu_0, \log \sigma_0) \propto 1$. It was shown in \citep{bardenet2017markov} that FlyMC algorithm behaves erratically in this case, as ``bright" data points with large values are rarely updated and they drive samples away from the target distribution. Consequently the chain mixes very slowly. One important commonality FlyMC shares with EWSG is that in each iteration, both algorithms select a subset of data in a non-uniform fashion. Therefore, it is interesting to investigate the performance of EWSG in this misspecified model.

For FlyMC\footnote{https://github.com/rbardenet/2017\-JMLR-MCMCForTallData}, a tight lower bound based on Taylor's expansion is used to minimize ``bright" data points used per iteration. At each iteration, 10\% data points are resampled and turned ``on/off" accordingly and the step size is adaptively adjusted. FlyMC algorithm is run for 10000 iterations. Figure \ref{fig:histogram} shows the histogram of number of data points used in each iteration for FlyMC algorithm. On average, FlyMC consumes $10.9\%$ of all data points per iteration. For fair comparison, the minibatch size of EWSG is hence set $10^5 \times 10.9\% = 10900$ and we run EWSG for 1090 data passes. We set step size $h=1\times10^{-4}$ and friction coefficient $\gamma=300$ for EWSG. An isotropic random walk Metropolis Hastings (MH) is also run for sufficiently long and serves as the ground truth. 

Figure \ref{fig:autocorrelation} shows the autocorrelation of three algorithms. The autocorrelation of FlyMC decays very slowly, samples that are even 500 iterations away still show strong correlation. The autocorrelation of EWSG, on the other hand, decays much faster, suggesting EWSG explores parameter space efficiently than FlyMC does. Figure \ref{fig:EWSG_samples} and \ref{fig:FlyMC_samples} show the samples (the first 1000 samples are discarded as burn-in) generated by EWSG and FlyMC respectively. The samples of EWSG center around the mode of the target distribution while the samples of FlyMC are still far away from the true posterior. The experiment shows EWGS works quite well even in misspecified models, and hence is an effective candidate in combining importance sampling with scalable Bayesian inference.

\begin{figure}[h!]
    \begin{subfigure}[t]{0.45\textwidth}
		\centering
		\includegraphics[width=\textwidth]{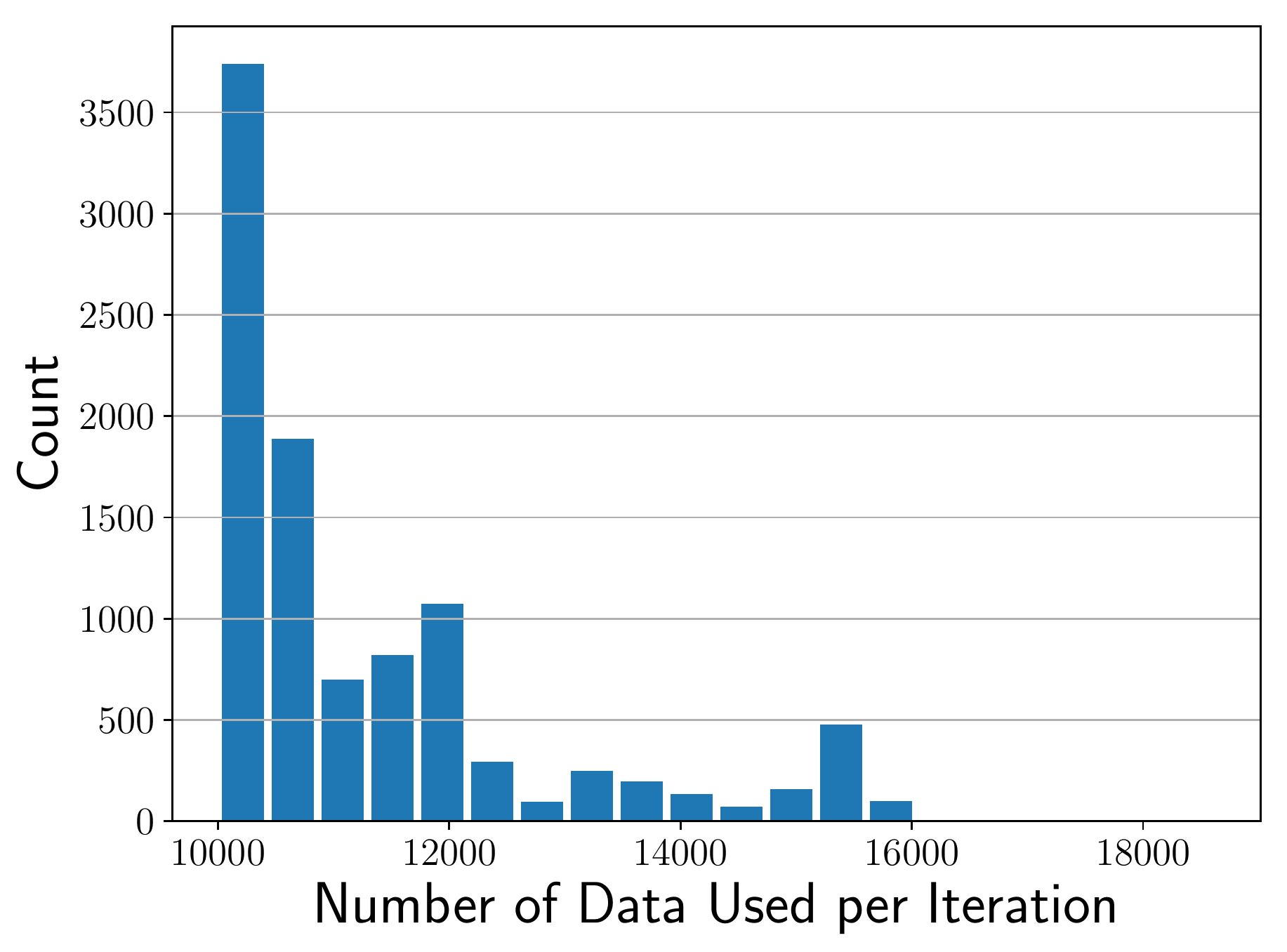}
		\caption{Histogram} \label{fig:histogram}
	\end{subfigure}
    \begin{subfigure}[t]{0.45\textwidth}
		\centering
		\includegraphics[width=\textwidth]{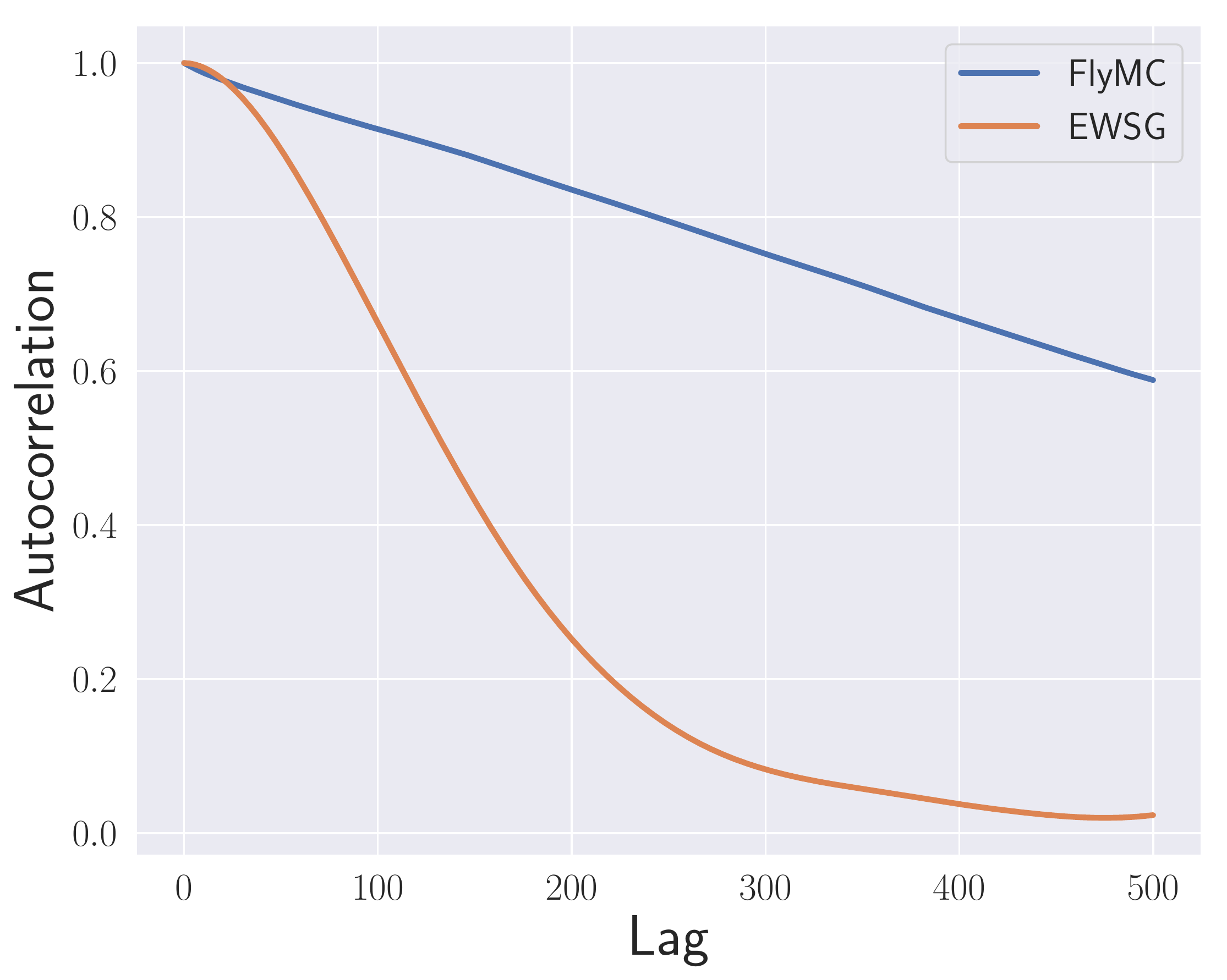}
		\caption{Autocorrelation} \label{fig:autocorrelation}
	\end{subfigure}

    \begin{subfigure}[t]{0.45\textwidth}
		\centering
		\includegraphics[width=\textwidth]{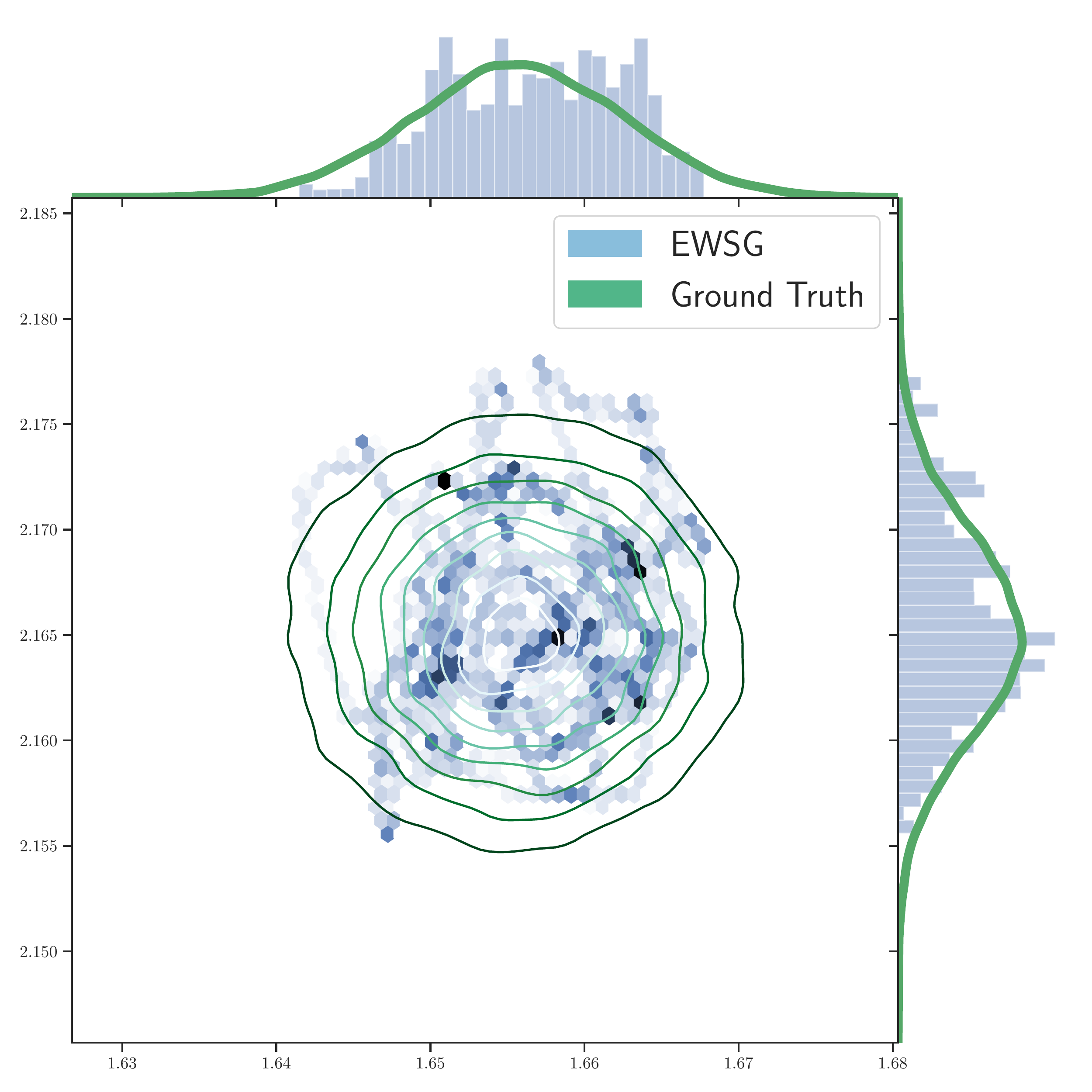}
		\caption{Samples of EWSG} \label{fig:EWSG_samples}
	\end{subfigure}
    \begin{subfigure}[t]{0.45\textwidth}
		\centering
		\includegraphics[width=\textwidth]{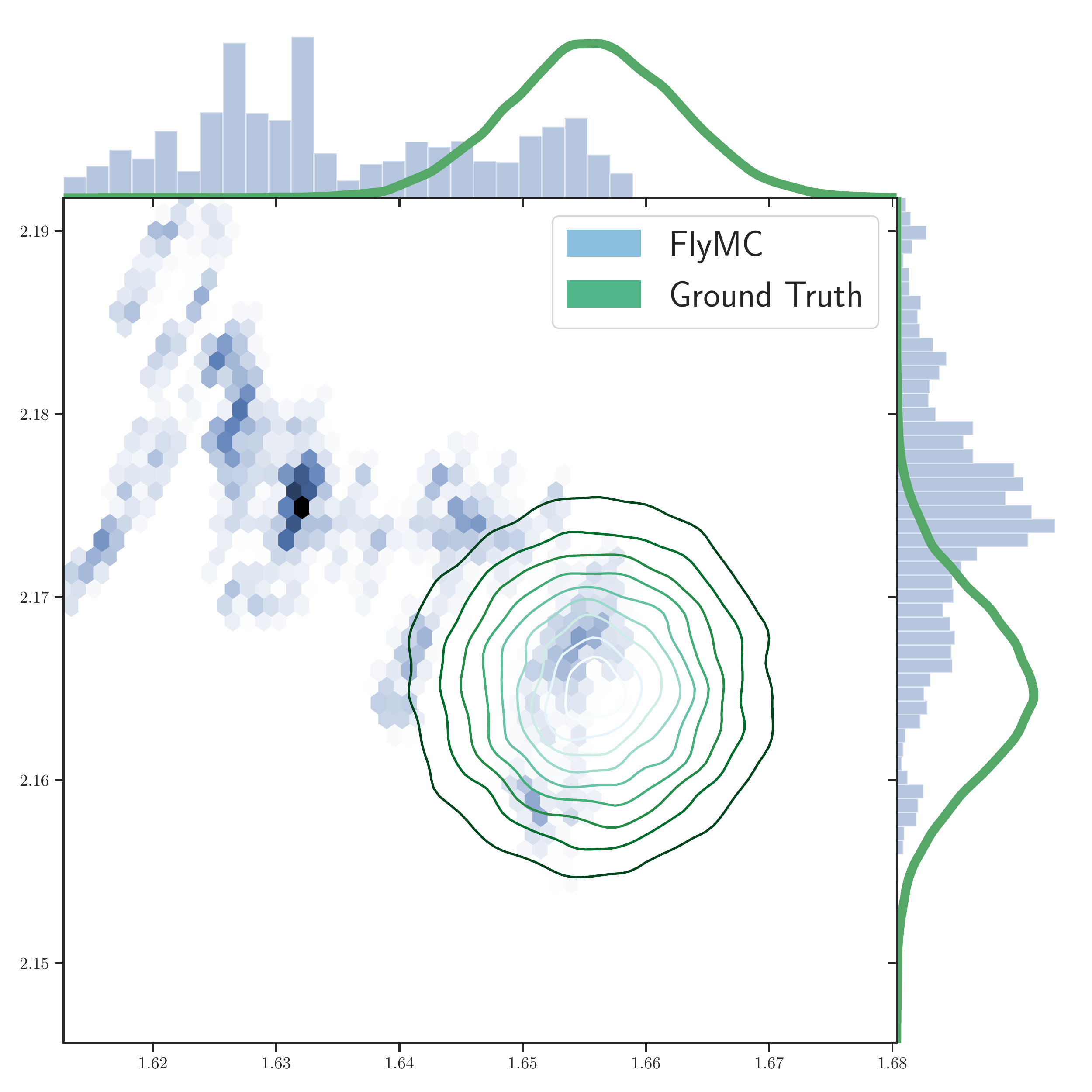}
		\caption{Samples of FlyMC} \label{fig:FlyMC_samples}
	\end{subfigure}
	
    \caption{(a) Histogram of data used in each iteration for FlyMC algorithm. (b) Autocorrelation plot of FlyMC, EWSG and MH. (c) Samples of EWSG. (d) Samples of FlyMC.}
\end{figure}

\subsection{Additional Results of BNN Experiment}\label{subsec:additional-BNN}
We report the test error of various SG-MCMC methods after 200 epochs in Table \ref{tab:test_error}. For both MLP and CNN architecture, EWSG outperforms its uniform counterpart SGHMC as well as other benchmarks SGLD, pSGLD and CP-SGHMC. The results clearly demonstrate the effectiveness of the proposed EWSG on deep models.
\begin{table}[!h]
    \centering
    \caption{Test error (mean $\pm$ standard deviation) after 200 epoches.}\label{tab:test_error}
    \begin{tabular}{ccc}
        \Xhline{3\arrayrulewidth}
        Method  & Test Error(\%), MLP &  Test Error(\%), CNN    \\
        \hline
        SGLD    &  1.976 $\pm$ 0.055     &    0.848 $\pm$ 0.060       \\
        \hline 
        pSGLD   &  1.821 $\pm$ 0.061    &    0.860 $\pm$ 0.052       \\
        \hline
        SGHMC  &  1.833 $\pm$ 0.073       &   0.778 $\pm$ 0.040    \\
        \hline 
        CP-SGHMC &  1.835 $\pm$ 0.047   &   0.772 $\pm$ 0.055                \\
        \hline 
        EWSG &    \textbf{1.793} $\pm$ \textbf{0.100} &  \textbf{0.753 $\pm$ 0.035} \\ 
        \Xhline{3\arrayrulewidth}
    \end{tabular}
\end{table}

\subsection{Additional Experiment on BNN: Tuning $M$} \label{subsec:BNN-M}
In each iteration of EWSG, we run an index Markov chain of length $M$ and select a ``good" minibatch to estimate gradient, therefore EWSG essentially uses $b \times (M + 1)$ data points per iteration where $b$ is minibatch size. How does EWSG compare with its uniform gradient subsampling counterpart with a larger minibatch size ($b \times (M + 1)$)?

We empirically answer this question in the context of BNN with MLP architecture. We use the same step size for SGHMC and EWSG and experiment a large range of values of minibatch size $b$ and index chain length $M$. Each algorithm is run for 200 data passes and 10 independent samples are drawn to estimate test error. The results are shown in Table \ref{tab:fine_compare}. We find that EWSG beats SGHMC with larger minibatch in 8 out of 9 comparison groups, which suggests in general EWSG could be a better way to consuming data compared to increasing minibatch size and may shed light on other areas where stochastic gradient methods are used (e.g. optimization).

\begin{table}[h!]
    \centering
    \begin{tabular}{c|ccc}
    \hline
    $b$ & \small $M+1=2$   & \small $M+1=5$  & \small $M+1=10$ \\ 
    \hline \\[-0.5em]
    $100$  & \shortstack{ \textbf{\ewsg{1.86\%}} \\ \SGHMC{1.94\%}} & \shortstack{ \textbf{\ewsg{1.83\%}} \\ \SGHMC{1.92\%}} &    \shortstack{\textbf{\ewsg{1.80\%}} \\ \SGHMC{1.97\%}} \\ \hline \\[-0.5em]
    $200$  & \shortstack{ \ewsg{1.90\%} \\ \textbf{\SGHMC{1.87}}\%} & \shortstack{ \textbf{\ewsg{1.87\%}} \\ \SGHMC{1.97\%}} &    \shortstack{ \textbf{\ewsg{1.80\%}} \\ \SGHMC{2.07\%}}              \\\hline\\[-0.5em]
    $500$  & \shortstack{ \textbf{\ewsg{1.79\%}} \\ \SGHMC{1.97\%}}  & \shortstack{ \textbf{\ewsg{2.01\%}} \\ \SGHMC{2.17\%}}  & \shortstack{ \textbf{\ewsg{2.36\%}} \\ \SGHMC{2.37\%}}              \\
    \hline 
     \end{tabular}\\
    \caption{Test errors of \ewsg{EWSG} (top of each cell) and \SGHMC{SGHMC} (bottom of each cell) after 200 epoches. $b$ is minibatch size for \ewsg{EWSG}, and minibatch size of \SGHMC{SGHMC} is set as $b\times(M+1)$ to ensure the same number of data used per parameter update for both algorithms. Step size is set $h=\frac{10}{b(M+1)}$ as suggested in \citep{chen2014stochastic}, different from that used to produce Table \ref{tab:test_error}. Results with smaller test error is highlighted in boldface.}\label{tab:fine_compare}
\end{table}

\section{EWSG does not necessarily change the speed of convergence significantly} \label{sec:speed}
Changing the weights of stochastic gradient from uniform to non-uniform, as we saw, can increase the statistical accuracy of the sampling; however, it does not necessarily increase or decrease the speed of convergence to the (altered) limiting distribution. Numerical examples already demonstrated this fact, but on the theoretical side, we note the non-asymptotic bound provided by Theorem \ref{thm:mse} may not be tight in terms of the speed of convergence due to its generality. Therefore, here we quantify the convergence speed on a simple quadratic example:

Consider $V_i(\theta)=\frac{1}{n}(\theta-\mu_i)^2/2$ where $\mu_i$'s are constant scalars. Assume without loss of generality that $\sum_i \mu_i = 0$, and thus $V(\theta)=\sum_{i=1}^n V_i(\theta) = \theta^2/2 + \text{some constant}$. We will show the convergence speed of $\mathbb{E}\theta$ is comparable for uniform and a class of non-uniform SG-MCMC (including EWSG) applied to second-order Langevin equation (overdamped Langevin will be easier and thus omitted):

\begin{theorem}
    Consider, for $0<\gamma<2$, respectively SGHMC and EWSG,
    \[
    \begin{cases}
        \theta'_{k+1} &= \theta'_k + h r'_k \\
        r'_{k+1} &= r'_k - h \gamma r'_k - h (\theta'_k - \mu_{I'_k}) + \sqrt{h}\sigma \xi'_{k+1}
    \end{cases}
    \]
    and
    \[
    \begin{cases}
        \theta_{k+1} &= \theta_k + h r_k \\
        r_{k+1} &= r_k - h \gamma r_k - h (\theta_k - \mu_{I_k}) + \sqrt{h}\sigma \xi_{k+1}
    \end{cases} ,
    \]
    where $I'_k$ are i.i.d. uniform random variable on $[n]$, $I_k$ are $[\theta, r]$ dependent random variable on $[n]$ satisfying $\mathbb{P}(I_k=i)=1/n+\mathcal{O}(h^p)$, and $\xi_{k+1},\xi'_{k+1}$ are standard i.i.d. Gaussian random variables. Denote by $\bar{\theta'}_k = \mathbb{E}\theta'_k$, $\bar{r'}_k = \mathbb{E}r'_k$,  $\bar{\theta}_k = \mathbb{E}\theta_k$, $\bar{r}_k = \mathbb{E}r_k$, $x'_k=[\bar{\theta'}_k, \bar{r'}_k]^T$, and $x_k=[\bar{\theta}_k, \bar{r}_k]^T$, then
    \begin{equation}
        x'_k=(I+A h)^k x'_0,    \quad \text{where } A=\begin{bmatrix} 0 & 1 \\ -1 & -\gamma \end{bmatrix},
        \label{eq_jdkgboirqyuhbgolqu124erbgoulb}
    \end{equation}
    for small enough $h$, $\|x'_k\|$ converges to 0 exponentially with $k\rightarrow\infty$, and $x_k$ converges at a comparable speed in the sense that $\|x_k-x'_k\|=\mathcal{O}(h^p)$ if $x_0=x'_0$.
\end{theorem}
\begin{proof}
    Taking the expectation of the $[\theta',r']$ iteration and using the fact that $\sum_i \mu_i=0$ and hence $\mathbb{E}\mu_{I'_k}=0$, one easily obtains \eqref{eq_jdkgboirqyuhbgolqu124erbgoulb}. The geometric convergence of $x'_k$ thus follows from the fact that eigenvalues of $I+Ah$ have less than 1 modulus for small enough $h$.

    Let $e_k = [0, \mathbb{E}\mu_{I_k}]^T$ and then
    \[
        e_k = [0, \sum_{i=1}^n \mathbb{P}(I_k=i) \mu_i]^T = [0, \mathcal{O}(h^p)]^T
    \]
    Now we take the expectation of both sides of the $[\theta,r]$ iteration and obtain $x_{k+1}=(I+Ah) x_k + h e_k$. Therefore
    \begin{align*}
        x_k =& (I+Ah)^k x_0 + (I+Ah)^{k-1} h e_0 + \cdots + (I+Ah) h e_{k-2} + h e_{k-1}\\
        =& x'_k + h \big( (I+Ah)^{k-1} e_0 + \cdots + (I+Ah) e_{k-2} + e_{k-1} \big)
    \end{align*}
    To bound the difference, note $I+Ah$ is diagonalizable with complex eigenvalues $\lambda_{1,2}$ satisfying
    \[
    |\lambda_1|=|\lambda_2|=\sqrt{1-h\gamma+h^2}=1-\gamma h/2+\mathcal{O}(h^2).
    \]
    Projecting $e_j$ to the corresponding eigenspaces via $e_j=v_{1,j}+v_{2,j}$, we can get
    \begin{align*}
        h \| (I+Ah)^{k-1} e_0 + \cdots + e_{k-1} \| 
        & \leq h \left( \| (I+Ah)^{k-1} e_0 \| + \cdots + \|e_{k-1} \| \right) \\
        & = h \left( |\lambda_1|^{k-1} \|v_{1,0}\| + |\lambda_2|^{k-1} \|v_{2,0}\| + \cdots + \|v_{1,k-1}\| + \|v_{2,k-1}\| \right) \\
        & \leq h C h^p (|\lambda_1|^{k-1} + \cdots + 1)
        = h C h^p \frac{1-|\lambda_1|^k}{1-|\lambda_1|} \leq h C h^p \frac{1}{1-|\lambda_1|} \\
        & \leq \hat{C} h^p
    \end{align*}
    for some constant $C$ and $\hat{C}$.
\end{proof}

Important to note is, although this is already a nonlinear example for EWSG (as nonlinearity enters through the $\mu_{I_k}$ term), it is a linear example for SGHMC. For the fully nonlinear cases, a tight quantification of EWSG's convergence speed remains to be an open theoretical challenge (a loose quantification is already given by the general Theorem \ref{thm:mse}).

\end{document}